\documentclass[11pt]{article}
\usepackage{verbatim,color}
\usepackage{amsthm,amsmath,amssymb,amsfonts}
\usepackage{graphicx}
\usepackage{epstopdf}   
\usepackage{bm}
\usepackage{natbib}
\bibpunct{(}{)}{;}{a}{}{,}
\usepackage[noline, ruled,boxed]{algorithm2e}
\usepackage[tight]{subfigure}
\usepackage[tableposition=top]{caption}
\usepackage{booktabs}
\usepackage[flushleft]{threeparttable}
\usepackage{arydshln}  
\usepackage{setspace}
\usepackage{afterpage}

\newtheorem{theorem}{Theorem}[section]

\newtheorem{lemma}[theorem]{Lemma}
\newtheorem{definition}[theorem]{Definition}

\def\Vcal{\mathcal V}

\def\bE{\mathbb E}
\def\bI{\mathbb I}

\begin{document}

\def\spacingset#1{\renewcommand{\baselinestretch}%
{#1}\small\normalsize} \spacingset{1}


\title{\bf Causal nearest neighbor rules for optimal treatment regimes}
\author{Xin Zhou\\[5pt]
{Departments of Biostatistics and Epidemiology}\\
{Harvard T.H. Chan School of Public Health}\\
{Boston, Massachusetts 02115, U.S.A.}
\\[20pt]
Michael R. Kosorok\\[5pt]
{Department of Biostatistics}\\
{University of North Carolina at Chapel Hill}\\
{Chapel Hill, North Carolina 27599, U.S.A.}
\\[1pt]
}
\maketitle

\begin{abstract}
The estimation of optimal treatment regimes is of considerable interest to precision medicine.
In this work, we propose a causal $k$-nearest neighbor method to estimate the optimal treatment regime. The method roots in the framework of causal inference, and estimates the causal treatment effects within the nearest neighborhood.
Although the method is simple, it possesses nice theoretical properties. We show that the causal $k$-nearest neighbor regime is universally consistent. That is, the causal $k$-nearest neighbor regime will eventually learn the optimal treatment regime as the sample size increases. We also establish its convergence rate. However, the causal $k$-nearest neighbor regime may suffer from the curse of dimensionality, \textit{i.e.}, performance deteriorates as dimensionality increases. To alleviate this problem, we develop an adaptive causal $k$-nearest neighbor method to perform metric selection and variable selection simultaneously. The performance of the proposed methods is illustrated in simulation studies and in an analysis of a chronic depression clinical trial.
\end{abstract}

\noindent%

\section{Introduction}
Precision medicine has recently gained much attention in treating complex diseases, such as cancer and mental disorders. The purpose of precision medicine is to tailor treatments to individual patients to maximize treatment benefit and safety in health care.
Modern precision medicine is different from the traditional ``one-size-fits-all'' approach, which does not rigorously take into account the treatment heterogeneity.

A major component of precision medicine is the treatment selection rule, or optimal treatment regime. A treatment regime is a decision rule that assigns a treatment to a patient based on his or her clinical or medical characteristics. A large number of approaches have been developed to estimate optimal treatment regimes based on data from clinical trials or observational studies (see \citet{Murphy:Design2005, Qian:ITR2011, Zhang:RobustITR2012, Taylor2015:QRF} and references therein). Most of these methods are regression-based. They model the conditional mean outcomes, and obtain the estimated treatment regime by comparing the regression estimates.

Several researchers have applied classification methods to optimal treatment regimes. For example, \citet{Zhao:OWL2012} viewed the treatment regime estimation as a weighted classification problem, and proposed outcome weighted learning to construct an optimal treatment regime to optimize the observed clinical outcome directly. Recently, \citet{Zhou2015:RWL} proposed residual weighted learning, which uses residuals to replace outcomes, to improve finite sample performance of outcome weighted learning.
\citet{Zhang2012:ClassificationITR} also proposed a general framework to make use of weighted classification methods to generate treatment regimes. As an illustrating example, they constructed a weighted classification problem through a doubly robust augmented inverse probability weighted estimator of the conditional mean outcome, and used classification and regression trees \citep{Breiman:CART84} to produce interpretable regimes.

The $k$-nearest neighbor rule is a simple and intuitively appealing classification approach, where a subject is classified by a majority vote of its neighbors. Since its conception \citep{Fix1951:KNN}, it has attracted many researchers, and retains its popularity today \citep{Stone1977:NN, Hastie1996:DAN, Wager2015:HeteroForest}.
The rationale of nearest neighbor rules is that close covariate vectors share similar properties more often than not.

In this article, we propose a causal $k$-nearest neighbor method for optimal treatment regimes. The method roots in the framework of causal inference, and compares the causal treatment effects within the nearest neighborhood.
Although the method is simple, it possesses nice theoretical properties. Firstly, we show that the causal $k$-nearest neighbor regime is universally consistent.
Without knowing any specifics about the distribution underlying the data, a universally consistent treatment regime would eventually learn the Bayes regime when the sample size approaches infinity. Secondly, we establish its convergence rate. The convergence rate is as high as $n^{-1/2}$ with appropriately chosen $k$ if the dimension of covariates is 1 or 2, and the rate is $n^{-2/(p+2)}$ for dimension $p\geq3$.

Similar to the nearest neighbor rule for classification, the causal $k$-nearest neighbor regime suffers from the curse of dimensionality, \textit{i.e.}, performance deteriorates as dimensionality increases. To alleviate this problem, we propose an adaptive causal $k$-nearest neighbor method, where the distance metric is adaptively determined from the data. Through adaptive metric selection, this adaptive method performs variable selection implicitly. The superior performance of the adaptive causal $k$-nearest neighbor regime over the original causal $k$-nearest neighbor regime is illustrated in the simulation studies.
In practical settings, we recommend the adaptive causal $k$-nearest neighbor method.

\section{Methods}
\subsection{Causal nearest neighbor rules}
Consider a randomized clinical trial with $L$ treatment arms. Let $R\in\mathcal{R}$ denote the observed clinical outcome, $A\in\mathcal{A}=\{1,\ldots,L\}$ denote the treatment assignment received by the patient, and ${X}=(X_1,\ldots,X_p)^T\in\mathcal{X}\subset\mathbb{R}^p$, where $\mathcal{X}$ is compact, denote the patient's clinical covariates. Assume that larger values of $R$ are preferred. Let $\pi_{\ell}({x}) = \textrm{pr}(A=\ell|{X}={x})$ denote the probability of being assigned treatment $\ell$ for a patient with clinical covariates ${x}$. This probability is predefined in the design.

We then introduce the potential outcomes framework to formally identify the optimal treatment regime. The potential outcomes, denoted $R^*(1),\cdots,R^*(L)$, are defined as the outcomes that would be observed were a patient to receive treatment $1,\cdots,L$, respectively \citep{Robins:PS1986}. As in the literature of potential outcomes, we require the following assumptions. The first one is the consistency assumption \citep{Robins1994:Noncomp}: the potential outcomes and the observed outcomes agree, \textit{i.e.}, $R = \sum_{\ell=1}^LR^*(\ell)\bI(A=\ell)$. We also assume that conditional on covariates ${X}$, the potential outcomes $\{R^*(1), \cdots,R^*(L)\}$ are independent of the treatment assignment $A$ that has been actually received. This is called the assumption of no unmeasured confounders. It always holds in a randomized clinical trial.

A treatment regime $d$ is a function from clinical covariates ${X}$ to the treatment assignment $A$. For a treatment regime $d$, we can thus define its potential outcome $R^*(d) =  \sum_{\ell=1}^LR^*(\ell)\bI(d({X})=\ell)$. It would be the observed outcome if a patient from the population were to be assigned treatment according to regime $d$.
The expected potential outcome under any regime $d$, given as $\Vcal(d)=\bE(R^*(d))$, is called the value function associated with regime $d$.
An optimal regime $d^*$ is a regime that maximizes $\Vcal(d)$. The regime $d^*$ is also called the Bayes regime.
There is a positivity assumption that $\pi_{\ell}({X})>0$ almost everywhere for any $\ell\in\mathcal{A}$. That is, any treatment option must be represented in the data in order to estimate an optimal regime. For simplicity, let $m_{\ell}({x})=\bE(R^{\ast}(\ell)|{X}={x})$.
It is easy to obtain that
\begin{equation} \label{eq:bayesknn}
d^*({x})=\textrm{argmax}_{\ell\in\{1,\ldots,L\}} m_{\ell}({x}).
\end{equation}
Note that $m_{\ell}({x})=\bE(R^{\ast}(\ell)|{X}={x})=\bE(R|{X}={x},A=\ell)$ by the consistency and no-unmeasured-confounders assumptions. It is identifiable in the observed data.

The $k$-nearest neighbor rule is a nonparametric method used for classification and regression \citep{Fix1951:KNN}. In this article, we apply the nearest neighbor rule to optimal treatment regimes. The idea is simple. We use the nearest neighbor algorithm to find a neighborhood of ${x}$ in $\mathcal{X}$, then estimate $m_{\ell}({x})$ for each arm in this neighborhood, and plug into (\ref{eq:bayesknn}) to get the nearest neighbor estimate for the optimal treatment regime. Similar procedures are proposed in the recent literature for tree-based nonparametric approaches \citep{Athey2016:Partitioning,Wager2015:HeteroForest}, where they target a partition in $\mathcal{X}$ to estimate the treatment heterogeneity.

Let $D_n=\{({X}_i, A_i, R_i): i=1,\ldots,n\}$ denote the observed data.
We fix ${x}\in\mathcal{X}$, and reorder the observed data $D_n$ according to increasing values of $||{X}_i-{x}||$. The reordered data sequence is denoted by
\begin{equation*}
\Big({X}_{(1,n)}({x}), A_{(1,n)}({x}), R_{(1,n)}({x})\Big),\ldots,\Big({X}_{(n,n)}({x}), A_{(n,n)}({x}), R_{(n,n)}({x})\Big).
\end{equation*}
Thus ${X}_{(1,n)}({x}),\ldots,{X}_{(k,n)}({x})$ are the $k$ nearest neighbors of ${x}$. 
$m_{\ell}({x})$ can be approximated in the $k$-nearest neighborhood of $x$ by
\begin{equation} \label{eq:knnrulej}
\hat{m}_\ell({x}) = \sum_{i=1}^n W_{n,i}^\ell({x}) R_{(i,n)}({x}), \, \textrm{where}\,
W_{n,i}^\ell({x})=\left\{
\begin{array}{cc}
\frac{{\bI(A_{(i,n)}({x})=\ell)}/{\pi_\ell\left({X}_{(i,n)}({x})\right)}}{\sum_{j=1}^k{\bI(A_{(j,n)}({x})=\ell)}/{\pi_\ell\left({X}_{(j,n)}({x})\right)}}
& {\rm if\ } i\leq k,\\
0 & {\rm if\ } i>k.
\end{array}
\right.
\end{equation}
and $\bI(\cdot)$ is the indicator function, as suggested in \citet{Murphy:Design2005}.
Here we define $0/0=0$.
Let $r({x})$ be the distance of the $k$th nearest neighbor to $x$, and define $S_{{x},\epsilon}=\{z\in\mathcal{X}:||z-x||\leq \epsilon\}$. It is straightforward through the consistency and no-unmeasured-confounders assumptions to see that $\hat{m}_\ell({x})$ is an unbiased estimator for
$\bE(R^{\ast}(\ell)|X\in S_{{x},r(x)})$. Hence $\hat{m}_\ell({x})$ is a reasonable approximation to $m_{\ell}({x})$.
Then the plug-in estimate of the Bayes regime in (\ref{eq:bayesknn}) is
\begin{equation} \label{eq:knnregimerule}
d^{CNN}({x}) = \textrm{argmax}_{\ell\in\{1,\ldots,L\}} \hat{m}_\ell({x}).
\end{equation}
This is called the causal $k$-nearest neighbor regime because of its close relationship to causal effects.

We need to address the problem of distance ties, \textit{i.e.}, when $||{x}-{X}_i||=||{x}-{X}_j||$ for some $i\neq j$.
\citet[Section 11.2]{Devroye:PatternRecog1996} discussed several methods for breaking distance ties.
In practical use, we adopt the tie-breaking method used in \citet{Stone1977:NN}. Subjects who have the same distance from ${x}$ as the $k$th nearest neighbor are averaged on the outcome $R$. We denote the distance of the $k$th nearest neighbor to ${x}$ by $\rho_k({x})$, and define the sets $A_k({x}):=\{i:||{x}-{X}_i||<\rho_k({x})\}$ and $B_k({x}):=\{i:||{x}-{X}_i||=\rho_k({x})\}$.
The revised rule of (2) in the main paper is as follows:
\begin{equation} \label{eq:knnstonerulej}
\tilde{m}_\ell({x}) = \frac{\sum_{i\in A_k({x})}R_{(i,n)}({x})\frac{\bI(A_{(i,n)}({x})=\ell)}{\pi_\ell\left({X}_{(i,n)}({x})\right)}+\frac{k-|A_k({x})|}{|B_k({x})|}\sum_{i\in B_k({x})}R_{(i,n)}({x})\frac{\bI(A_{(i,n)}({x})=\ell)}{\pi_\ell\left({X}_{(i,n)}({x})\right)}}
{\sum_{i\in A_k({x})}\frac{\bI(A_{(i,n)}({x})=\ell)}{\pi_\ell\left({X}_{(i,n)}({x})\right)}+\frac{k-|A_k({x})|}{|B_k({x})|}\sum_{i\in B_k({x})}\frac{\bI(A_{(i,n)}({x})=\ell)}{\pi_\ell\left({X}_{(i,n)}({x})\right)}}.
\end{equation}
The corresponding causal nearest neighbor regime is the regime in (3) of the main paper after replacing $\hat{m}_\ell({x})$ with $\tilde{m}_\ell({x})$. This is not a strictly $k$-nearest neighbor rule when there are distance ties on the $k$th nearest neighbor, since the estimate uses more than $k$ neighbors.

The causal nearest neighbor regimes are based on local averaging. Here, $k$ is a tuning parameter. It is required that $k$ be small enough so that local changes of the distribution can be detected. On the other hand, $k$ needs to be large enough so that averaging over the arm is effective. We may tune this parameter by a cross validation procedure to balance the two requirements. 


In this article, we focus on applications in randomized clinical trials. However, the proposed methods can be easily extended to observational studies. We still require three assumptions (consistency, no unmeasured confounders and positivity). The only difference is that the assumption of no unmeasured confounders automatically holds in randomized clinical trials. In observational studies, it may hold when all relevant confounders have been measured, though this assumption cannot be verified in practice. One additional step for observational studies is to estimate the treatment allocation probabilities $\pi_{\ell}({x})$, which can be obtained through, for example, logistic regression.

To our knowledge, no nearest neighbor related methods were applied in optimal treatment regimes. \citet{Wager2015:HeteroForest} described a standard $k$-nearest neighbor matching procedure to estimate heterogeneous treatment effects. \citet{Wager2015:HeteroForest} used a different estimator
\begin{equation} \label{nnestimator}
\bar{m}_\ell({x}) = \frac{1}{k}\sum_{i\in\mathcal{S}_{\ell}}R_i,
\end{equation}
where $\mathcal{S}_{\ell}$ is the set of $k$ nearest neighbors to $x$ in the treatment arm $\ell$. $\hat{m}_\ell({x})$ in (\ref{eq:knnrulej}) and $\bar{m}_\ell({x})$ in (\ref{nnestimator}) are two estimates of ${m}_\ell({x})$ from different perspectives. $\hat{m}_\ell({x})$ is an estimate of $\bE(R^{\ast}(\ell)|{X}={x})$, while $\bar{m}_\ell({x})$ is an estimate of $\bE(R|{X}={x},A=\ell)$. Our proposed causal $k$-nearest neighbor method is slightly distinct in two ways. First, the estimates $\hat{m}_\ell({x})$ are obtained from the same neighborhood of $x$. The subsequent comparison is more sensible. Second, for fairly large neighborhood, the inverse probability weighting estimator (\ref{eq:knnrulej}) corrects for variations $\pi_{\ell}({x})$ inside the neighborhood. This is particularly useful for applications in observational studies.


\subsection{Theoretical Properties}
In machine learning, a classification rule is called universally consistent if its expected error probability approaches the Bayes error probability, in probability or almost surely,
for any distribution underlying the data \citep{Devroye:PatternRecog1996}.
The $k$-nearest neighbor classification is the first to be proved to possess such universal consistency \citep{Stone1977:NN}. Here, we extend the concept of universal consistency to optimal treatment regimes.

\begin{definition}
Given a sequence $D_n$ of data, a regime $d_n$ is universally (weakly) consistent if $\lim_{n\rightarrow\infty}\mathcal{V}(d_n) = \mathcal{V}(d^*)$ in probability for any probability measure $P$ on $\mathcal{X}\times\mathcal{A}\times\mathcal{R}$, and universally strongly consistent if $\lim_{n\rightarrow\infty}\mathcal{V}(d_n) = \mathcal{V}(d^*)$ almost surely
for any probability measure $P$ on $\mathcal{X}\times\mathcal{A}\times\mathcal{R}$.
\end{definition}


Denote the probability measure for ${X}$ by $\mu$, and recall that $S_{{x},\epsilon}$ is the closed ball centered at ${x}$ of radius $\epsilon>0$. The collection of all ${x}$ with $\mu(S_{{x},\epsilon})>0$ for all $\epsilon>0$ is called the support of $\mu$ \citep{Cover1967:NN}. The set is denoted as $support(\mu)$.

The analysis of universal consistency requires some assumptions. 
\begin{enumerate}
\item[(A1)] There exists a constant $\zeta>0$ such that $\pi_\ell({x}) \geq \zeta$ for any ${x}\in support(\mu)$ and $\ell\in\{1,\ldots,L\}$;
\item[(A2)] $\bE|R| < \infty$;
\item[(A3)] Distance ties occur with probability zero in $\mu$.
\end{enumerate}
These assumptions are quite weak. Assumption (A1) is just the positivity assumption, and $\zeta$ can be obtained by design. Assumption (A2) is natural. This assumption is automatically satisfied for bounded outcomes, \textit{i.e.}, $|R|\leq M <\infty$ for some constant $M$.
Assumption (A3) is to avoid the messy problem of distance ties. 
When (A3) does not hold, we may add a small uniform variable $U\sim uniform(0,\epsilon)$ independent of $({X},A,R)$ to the vector ${X}$.
This causes the $(p+1)$-dimensional random vector ${X}'=({X},U)$ to satisfy Assumption (A3).
We may perform the $k$-nearest neighbor method on the modified data $D'_n=\{({X}'_i, A_i, R_i): i=1,\ldots,n\}$. Because of the independence of $U$, the corresponding conditional outcome $m'_{\ell}({x}') = \mathbb{E}(R|{X}'={x}',A=\ell)=m_{\ell}({x})$. Hence Assumption (A3) is reasonable, but at the cost of potentially compromising performance by introducing an artificial covariate to the regime.
When $\epsilon$ is very small, we actually break ties randomly. The difference with Stone's tie-breaking method is that Stone's method takes into account all subjects whose distance to ${x}$ equals that of the $k$th nearest neighbor, while the tie-breaking method here only picks one of them randomly. The remark following the proof of Theorem \ref{thm:knnconsistency} in Appendix A demonstrates that Stone's tie-breaking estimate in (\ref{eq:knnstonerulej}) is asymptotically better than the random tie-breaking method here.

The following theorem shows universal consistency of the causal nearest neighbor regime. The proofs of theorems are provided in Appendix A.
\begin{theorem} \label{thm:knnconsistency}
For any distribution $P$ for $({X},A,R)$ satisfying assumptions (A1)$\sim$(A3),

(i) the causal $k$-nearest neighbor regime in (\ref{eq:knnregimerule}) is universally weakly consistent if
$k\rightarrow\infty$ and $k/n\rightarrow0$;

(ii) the causal regime in (\ref{eq:knnregimerule}) is universally strongly consistent if $k/\log(n)\rightarrow\infty$ and $k/n\rightarrow0$.

If Assumption (A2) is tightened to $|R|\leq M<\infty$ for some constant $M$, the regime in (\ref{eq:knnregimerule}) is universally strongly consistent if $k\rightarrow\infty$ and $k/n\rightarrow0$.
\end{theorem}

The next natural question is whether the associated value of the causal $k$-nearest neighbor regime tends to the Bayes value at a specified rate. To establish the rate of convergence, we require stronger assumptions. 
\begin{enumerate}
\item[(A1$'$)] $\sum_{i=1}^nW_{n,i}^\ell({x})=1$ for all ${x}\in\mathcal{X}$ and $\ell=1,\ldots,L$, and there exists a constant $c$ such that $W_{n,i}^\ell({x})\leq c/k$ for all ${x}\in\mathcal{X}$, $i=1,\ldots,n$ and $\ell=1,\ldots,L$;
\item[(A2$'$)] there exists a constant $\sigma^2$ such that $\sigma^2_\ell({x})=\textrm{var}(R|{X}={x},A=\ell)\leq\sigma^2$ for all ${x}\in support(\mu)$ and $\ell=1,\ldots,L$;
\item[(A3$'$)] distance ties occur with probability zero in $\mu$, and the support of $\mu$ is compact with diameter $2\rho$;
\item[(A4$'$)] $m_\ell$'s are Lipschitz continuous, \textit{i.e.}, there exists a constant $C>0$ such that $|m_\ell({x})-m_\ell({x}')|\leq C||{x}-{x}'||$ for any ${x}$ and ${x}'$ in $\mathcal{X}$, and $\ell=1,\ldots,L$.
\end{enumerate}
Assumption (A1$'$) implies that randomization is not extremely skewed with respect to the covariates. Assumptions (A2$'$)$\sim$(A4$'$) are standard in the literature of nearest neighbor rules \citep{Gyorfi2002:Nonparametric}. The following theorem gives the convergence rate of causal $k$-nearest neighbor regimes. This theorem is proved in Appendix B.
\begin{theorem} \label{thm:knnconvergence}
For any distribution $P$ for $({X},A,R)$ satisfying Assumptions (A1$'$)$\sim$(A4$'$), there exists a sequence $k$ such that $k\rightarrow \infty$ and $k/n\rightarrow0$, and
\begin{equation*}
\bE\left\{\left(\mathcal{V}(d^*) - \mathcal{V}(d^{CNN})\right)^2\right\} = O(n^{-\beta}).
\end{equation*}
When $p=1$, $\beta=1/2$; when $p=2$, $\beta$ can be arbitrarily close to $1/2$; when $p\geq3$, $\beta=2/(p+2)$.
\end{theorem}
The rate of convergence is as high as $n^{-1/2}$ if the dimensionality $p$ is 1 or 2. When $p$ increases, the convergence rate decreases significantly. As with nearest neighbor rules in classification and regression, the causal $k$-nearest neighbor regime also suffers from the curse of dimensionality.

\subsection{Adaptive rules} \label{sec:adaptive}
The causal $k$-nearest neighbor regime is consistent as shown previously. However, it is well known that the curse of dimensionality can severely hurt nearest neighbor rules in finite samples. The rate of convergence in Theorem \ref{thm:knnconvergence} is slow when dimensionality is high. Hence appropriate variable selection may improve performance.
In this section, we propose an adaptive causal $k$-nearest neighbor method to estimate the optimal treatment regime, and to perform metric selection and variable selection simultaneously. 

Let $\Sigma=\textrm{diag}(\sigma_1^2, \ldots,\sigma_p^2)$
and use the distance metric $d({x}_1,{x}_2)=\{({x}_1-{x}_2)^T\Sigma({x}_1-{x}_2)\}^{1/2}$ to compute the distance between ${x}_1$ and ${x}_2$. $\sigma_j$ is the scaling factor for the $j$th covariate. Setting ${\sigma_j}=0$ is equivalent to discarding the $j$th covariate. We intend to set a large $\sigma_j^2$ if the $j$th covariate is important for treatment selection.

We apply the following univariate method to evaluate the importance of an individual covariate. It is related to a test statistic comparing two treatment regimes \citep{Murphy:Design2005}. One regime $d^j$ only involves the $j$th covariate; and the other $d^0$, called the non-informative regime, assigns all patients to the treatment with the largest estimated potential outcome $\hat{\bE}(R^{\ast}(\ell))=\sum_{i=1}^n\{R_i\bI(A_i=\ell)/\pi_{\ell}({X}_i)\}/\sum_{i=1}^n\{\bI(A_i=\ell)/\pi_{\ell}({X}_i)\}$. For a specific regime $d$, let $d_i$ be the treatment assignment for the $i$th subject according to $d$. The value function associated with $d$ is estimated by
\begin{equation} \label{eq:valueestimate}
\hat{\Vcal}(d) = \sum_{i=1}^n\left\{R_i{\bI(A_i=d_i)}/{\pi_{A_i}({X}_i)}\right\}\Big/\sum_{i=1}^n\left\{{\bI(A_i=d_i)}/{\pi_{A_i}({X}_i)}\right\}.
\end{equation}
For two regimes, $d^j$ and $d^0$, a consistent estimator of the variance of $\surd{n}(\hat{\Vcal}(d^j)-\hat{\Vcal}(d^0))$ is
\begin{equation} \small
\widehat{var}\left(\surd{n}(\hat{\Vcal}(d^j)-\hat{\Vcal}(d^0))\right) = \frac{1}{n}\sum_{i=1}^n\left\{\left(\frac{\bI(A_i=d^j_i)(R_i-\hat{\Vcal}(d^j))}{\pi_{A_i}({X}_i)}\right)^2+ \left(\frac{\bI(A_i=d^0_i)(R_i-\hat{\Vcal}(d^0))}{\pi_{A_i}({X}_i)}\right)^2\right\}. \label{eq:valuevariance}
\end{equation}
The statistic $T_j = {\surd{n}(\hat{\Vcal}(d^j)-\hat{\Vcal}(d^0))}\Big/{\left\{\widehat{var}\left(\surd{n}(\hat{\Vcal}(d^j)-\hat{\Vcal}(d^0))\right)\right\}^{1/2}}$
asymptotically has a standard normal distribution under the null hypothesis that ${\Vcal}(d^j)={\Vcal}(d^0)$ \citep{Murphy:Design2005}. When the statistic is greater than zero, regime $d^j$ is considered better than the non-informative regime $d^0$, otherwise $d^0$ is better.
The statistic $T_j$ reflects the importance of the $j$th covariate on optimal treatment regimes.
We estimate $d^j$ by the causal $k$-nearest neighbor method only using the $j$th covariate. 

We set $\sigma_j^2 = (T_j-\Delta)_+$ for each $j=1,\ldots,p$, where $\Delta\in\mathbb{R}$ is a predefined parameter and $(\cdot)_+$ is the positive part. The adaptive causal $k$-nearest neighbor regime follows the same procedure used in the causal $k$-nearest neighbor regime described above, except that the adaptive one uses the distance metric $d({x}_1,{x}_2)=\{({x}_1-{x}_2)^T\Sigma({x}_1-{x}_2)\}^{1/2}$ to compute the distance between ${x}_1$ and ${x}_2$. When $\Delta$ is very large (for example, $\Delta\rightarrow +\infty$), all $\sigma_j^2$ are zero, hence the adaptive regime degenerates to a non-informative regime. On the other hand, when $\Delta$ is very small (for example, $\Delta\rightarrow -\infty$), all $\sigma_j^2$ are almost identical, and the adaptive regime is equivalent to the causal $k$-nearest neighbor one.
Figure~\ref{knnfig:examples} illustrates the effects of $\Delta$ on the construction of $\Sigma$.

\begin{figure}
\centering
\subfigure[$\sigma_1^2=\sigma_2^2=\sigma_3^2=\sigma_4^2=\sigma_5^2=0$]{\label{knnfig:example1} \includegraphics[scale=0.5]{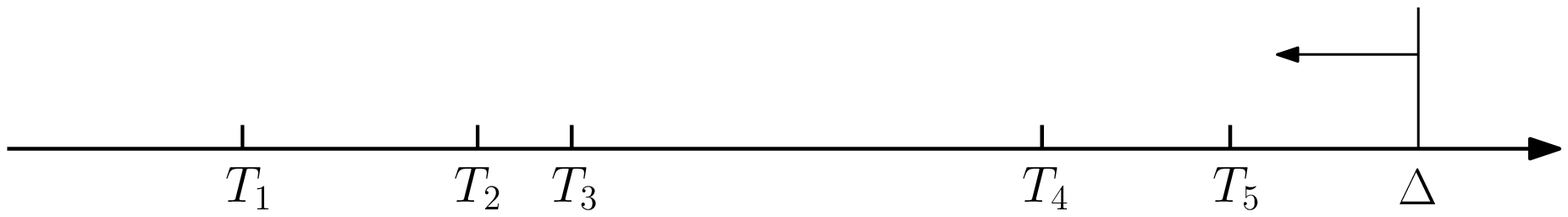}} \\[2ex]
\subfigure[$0=\sigma_1^2=\sigma_2^2=\sigma_3^2<\sigma_4^2<\sigma_5^2$]{\label{knnfig:example2} \includegraphics[scale=0.5]{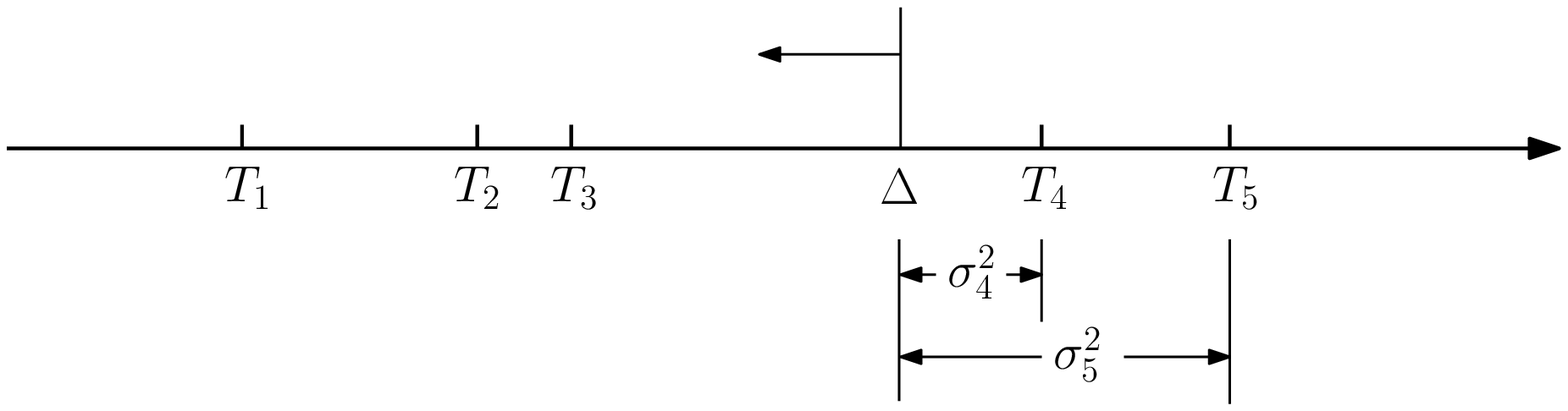}} \\[2ex]
\subfigure[$0<\sigma_1^2<\sigma_2^2<\sigma_3^2<\sigma_4^2<\sigma_5^2$]{\label{knnfig:example3} \includegraphics[scale=0.5]{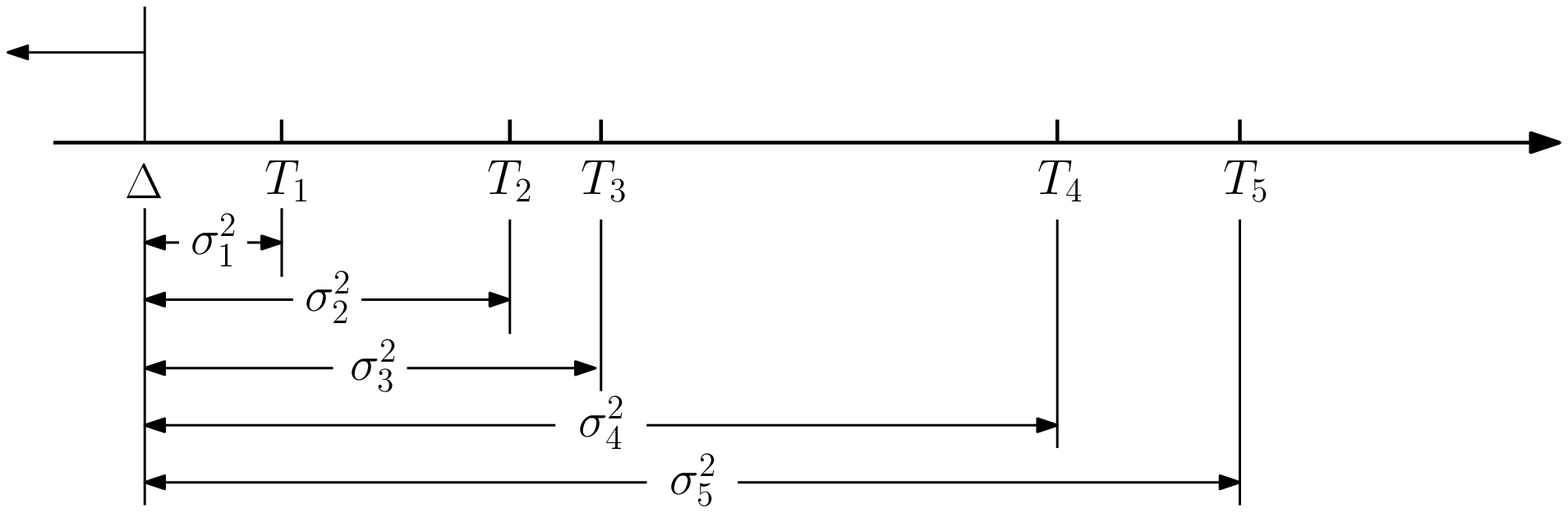}}
\caption{Three examples for construction of $\Sigma$ with different choices of $\Delta$. Suppose that there are five covariates with test statistics $T_1,\cdots,T_5$ from individual comparison tests. In example (a), $\Delta$ is greater than any test statistic, and hence $\sigma^2$'s are all zero. Under this situation, the adaptive causal $k$-nearest neighbor regime degenerates to a non-informative regime. When $\Delta$ decreases, some $\sigma^2$'s turn to positive from zero. In example (b), $\Delta$ is between $T_3$ and $T_4$. $\sigma_4^2$ and $\sigma_5^2$ are positive, and the first three are still zero. It is equivalent to throwing away the first three covariates in the analysis. When $\Delta$ continues to decrease, in example (c) $\Delta$ is smaller than any test statistic. All $\sigma^2$'s are non-zero. The adaptive causal $k$-nearest neighbor regime involves all five covariates. However, the fifth covariate contributes the most for the regime, and the first contributes the least.}
\label{knnfig:examples}
\end{figure}

Here is a summary of the adaptive causal nearest neighbor procedure:

1) Normalize each covariate to a similar scale.

2) Calculate $T_j$ and $\Sigma=\textrm{diag}(\sigma_1^2, \ldots,\sigma_p^2)$, where $\sigma_j^2 = (T_j-\Delta)_+$ and $j=1,\ldots,p$.

3) Use the metric $d({x}_1,{x}_2)=\left\{({x}_1-{x}_2)^T\Sigma({x}_1-{x}_2)\right\}^{1/2}$ to estimate a causal $k$-nearest neighbor regime.





The scaling at the first step is to avoid covariates in greater numeric
ranges dominating those in smaller numeric ranges.  We recommend linearly
scaling each covariate to the range $[-1, +1]$ or $[0, 1]$ \citep{Hsu2003:svmguide}.
For the adaptive causal $k$-nearest neighbor regime, there are two tuning parameters, $k$ and $\Delta$. We tune the parameters using 10-fold cross validation.

\section{Simulation studies}
We performed extensive simulations to evaluate empirical performance of the causal $k$-nearest neighbor and adaptive causal $k$-nearest neighbor methods.

We first considered simulations for two-arm data ($L=2$).
In the simulations, we generated $p$-dimensional vectors of clinical covariates. The first two covariates were independent Bernoulli random variables with success probability of $0.5$, and the remaining covariates were independent standard normal random variables $N(0,1)$. The treatment $A$ was generated from $\mathcal{A}=\{1, 2\}$ independently of ${X}$ with $P(A=1)=0.5$, \textit{i.e.,} $\pi_1({x})=\pi_2({x})=0.5$ for any ${x}\in\mathbb{R}^{p}$.
To mimic a well balanced trial, we generated simulation data such that $n_1/n=n_2/n=0.5$, where $n$ is the sample size of the data, $n_1$ and $n_2$ are the numbers of patients in treatment arm 1 and 2, respectively.
The response $R$ was normally distributed with mean $Q_0({x},a)$ and standard deviation 1. We considered three scenarios with different choices of $Q_0({x},a)$:
\begin{enumerate}
\item[(1)] $Q_0({x},1)=(1+0.5x_1+0.8x_2+x_3-0.5x_4+0.7x_5) + (0.3-0.2x_1-0.5x_3)$;

$Q_0({x},2)=(1+0.5x_1+0.8x_2+x_3-0.5x_4+0.7x_5) - (0.3-0.2x_1-0.5x_3)$.
\item[(2)] $Q_0({x},1)=(1+0.5x_1+0.8x_2+0.3x_3^2-0.5x_4^2+0.7x_5) + (0.3x_3 - 0.5x_4^2 + 0.4)$;

$Q_0({x},2)=(1+0.5x_1+0.8x_2+0.3x_3^2-0.5x_4^2+0.7x_5) - (0.3x_3 - 0.5x_4^2 + 0.4)$.
\item[(3)] $Q_0({x},1)=(1+0.5{x}_1+0.8{x}_2+0.3\tilde{x}_3-0.5\tilde{x}_4+0.7\tilde{x}_5) + (1 - \tilde{x}_3-\tilde{x}_4)$;

$Q_0({x},2)=(1+0.5{x}_1+0.8{x}_2+0.3\tilde{x}_3-0.5\tilde{x}_4+0.7\tilde{x}_5) - (1 - \tilde{x}_3-\tilde{x}_4)$,

where $\tilde{x}_j = \min(x_j^2,1)$, for $j=3,4,5$.
\end{enumerate}


We run the simulations for two different dimensions of covariates: low dimensional data ($p=5$) and moderate dimensional data ($p=25$). On low dimensional data ($p=5$), we compared empirical performances of the following seven methods: (1) $\ell_1$ penalized least squares proposed by \citet{Qian:ITR2011}; (2) Q-learning using random forests as described in \citet{Taylor2015:QRF}; 
(3) Residual weighted learning proposed in \citet{Zhou2015:RWL} using the linear kernel; (4) Residual weighted learning using the Gaussian kernel; (5) Augmented inverse probability weighted estimation proposed by \citet{Zhang2012:ClassificationITR}; (6) the causal $k$-nearest neighbor method; and (7) the proposed adaptive causal $k$-nearest neighbor method. When the dimension was moderate ($p=25$), two residual weighted learning methods were replaced with their variable selection counterparts \citep{Zhou2015:RWL}.

In the simulation studies, $\ell_1$ penalized least squares estimated a linear model on $(1,{X},A,{X}A)$ to approximate the conditional outcomes ${E}(R|{X},A)$, and also used the least absolute shrinkage and selection operator to carry out variable selection. The obtained regime was the treatment arm in which the conditional mean is larger.
Q-learning using random forests is nonparametric. The conditional outcomes ${E}(R|{X},A)$ were approximated using $({X},A)$ as input covariates in the random forests. The number of trees was set to 1000 as suggested in \citet{Taylor2015:QRF}.
Residual weighted learning is an improved method for outcome weighted learning \citep{Zhao:OWL2012}. Outcome weighted learning views the treatment selection as a weighted classification problem, and treats the original outcomes as weights. Residual weighted learning is similar except that outcomes are replaced with residuals of the outcome from a regression fit on covariates excluding treatment assignment. Residual weighted learning with the linear kernel estimates linear treatment regimes, while the one with the Gaussian kernel has the ability to detect nonlinear regimes. Residual weighted learning involves non-convex programming, and hence the computational cost is high.
For the augmented inverse probability weighted estimator, we first obtained the doubly robust version of the contrast function through linear regression, and then we let the propensity score be 0.5 and searched the optimal treatment regime using a classification and regression tree.

We applied 10-fold cross-validation for parameter tuning. The sample sizes were varied from $n=50$, $100$, $200$, $400$, to $800$ for each scenario. We repeated the simulation 500 times. For comparison, we generated a large test set with 10,000 subjects to evaluate performance. The comparison criterion was the value function of the estimated optimal treatment regime on the test set. Precisely, it is given by $\mathbb{P}_n^*[R\mathbb{I}(A=d({X}))/\pi_A({X})]/\mathbb{P}_n^*[\mathbb{I}(A=d({X}))/\pi_A({X})]$ \citep{Murphy:Design2005}, where $\mathbb{P}_n^*$ denotes the empirical average on the test data.

\begin{table}
\centering
\caption{Mean (standard deviation) of empirical value functions evaluated on the test set for Scenarios 1-3 when the dimension is low ($p=5$). The best value function for each scenario and sample size combination is in bold.}
\label{tab:valueknnlow}
\begin{threeparttable}
\begin{tabular}{lccccc}
\addlinespace
\toprule
& $n=50$ & $n=100$ & $n=200$ & $n=400$ & $n=800$ \\
\midrule
\multicolumn{6}{c}{Scenario 1 (Optimal value $2.09$)} \\
\midrule
$\ell_1$-PLS & \textbf{2.01 (0.08)} & \textbf{2.04 (0.06)} & \textbf{2.06 (0.04)} & \textbf{2.08 (0.02)} & \textbf{2.08 (0.01)} \\
Q-RF & 1.89 (0.08) & 1.96 (0.06) & 2.00 (0.04) & 2.02 (0.02) & 2.03 (0.01) \\
RWL-Linear & 1.96 (0.09) & 2.01 (0.06) & 2.05 (0.04) & 2.07 (0.02) & \textbf{2.08 (0.01)} \\
RWL-Gaussian & 1.97 (0.09) & 2.00 (0.07) & 2.03 (0.06) & 2.06 (0.03) & \textbf{2.08 (0.02)} \\
AIPWE & 1.92 (0.13) & 1.96 (0.10) & 2.00 (0.06) & 2.02 (0.04) & 2.03 (0.03) \\
CNN &  1.89 (0.10) & 1.95 (0.08) & 1.99 (0.06) & 2.02 (0.03) & 2.04 (0.02) \\
ACNN & 1.88 (0.13) & 1.94 (0.12) & 1.99 (0.08) & 2.02 (0.05) & 2.04 (0.03)  \\
\midrule
\multicolumn{6}{c}{Scenario 2 (Optimal value $1.95$)} \\
\midrule
$\ell_1$-PLS & 1.48 (0.08) & 1.53 (0.09) & 1.58 (0.09) & 1.64 (0.06) & 1.66 (0.03)  \\
Q-RF & 1.64 (0.10) & 1.74 (0.08) & 1.82 (0.05) & 1.87 (0.03) & 1.90 (0.02) \\
RWL-Linear & 1.55 (0.08) & 1.58 (0.07) & 1.61 (0.05) & 1.64 (0.04) & 1.66 (0.03) \\
RWL-Gaussian & 1.64 (0.11) & 1.71 (0.11) & 1.81 (0.08) & 1.86 (0.05) & 1.90 (0.02) \\
AIPWE & 1.63 (0.15) & 1.74 (0.13) & 1.81 (0.09) & 1.87 (0.05) & 1.90 (0.03)   \\
CNN & 1.64 (0.11) & 1.73 (0.09) & 1.81 (0.06) & 1.87 (0.03) & 1.90 (0.02) \\
ACNN & \textbf{1.65 (0.14)} & \textbf{1.76 (0.12)} & \textbf{1.84 (0.08)} & \textbf{1.89 (0.04)} & \textbf{1.92 (0.03)} \\
\midrule
\multicolumn{6}{c}{Scenario 3 (Optimal value $2.37$)} \\
\midrule
$\ell_1$-PLS &  1.88(0.03) & 1.89(0.03) & 1.89(0.03) & 1.89(0.04) & 1.90(0.03)  \\
Q-RF & 2.05 (0.08) & 2.14 (0.06) & 2.21 (0.04) & 2.26 (0.02) & 2.28 (0.01) \\
RWL-Linear & 1.93 (0.05) & 1.93 (0.05) & 1.96 (0.06) & 1.97 (0.06) & 1.98 (0.06) \\
RWL-Gaussian & 2.04 (0.09) & 2.13 (0.08) & 2.20 (0.06) & 2.26 (0.04) & 2.30 (0.02)  \\
AIPWE & 2.06 (0.13) & 2.17 (0.11) & 2.23 (0.05) & 2.26 (0.03) & 2.28 (0.02) \\
CNN & 2.02 (0.08) & 2.11 (0.06) & 2.18 (0.05) & 2.24 (0.03) & 2.28 (0.02) \\
ACNN &  \textbf{2.09 (0.11)} & \textbf{2.19 (0.08)} & \textbf{2.26 (0.06)} & \textbf{2.31 (0.04)} & \textbf{2.33 (0.02)}  \\
\bottomrule
\end{tabular}
\begin{tablenotes}
\item $\ell_1$-PLS, $\ell_1$ penalized least squares; Q-RF, Q-learning using random forests; RWL-Linear, residual weighted learning with linear kernel; RWL-Gaussian, residual weighted learning with Gaussian kernel; AIPWE, augmented inverse probability weighted estimation; CNN, causal $k$-nearest neighbor; ACNN, adaptive causal $k$-nearest neighbor.
\end{tablenotes}
\end{threeparttable}
\end{table}

The simulation results on the low dimensional data ($p=5$) are presented in Table~\ref{tab:valueknnlow}. Let $\mathcal{V}_{\ell}$, $\ell=1,\cdots,L$, be the value function when all subjects are sent to treatment $\ell$, and $\mathcal{V}^{\ast}$ be the optimal value function for simplicity. For Scenario 1, $\mathcal{V}_{1}=1.85$, $\mathcal{V}_{2}=1.45$, and $\mathcal{V}^{\ast}=2.09$. The optimal regime $d^{\ast}({x})$ is 1 if $0.2x_1+0.5x_3<0.3$, and 2 otherwise. 
The decision boundary was a linear combination of a binary covariate and a continuous covariate. $\ell_1$ penalized least squares performed very well since its model was correctly specified. Both residual weighted learning methods performed similarly to $\ell_1$ penalized least squares, especially when the sample size was large. Our proposed causal $k$-nearest neighbor and adaptive causal $k$-nearest neighbor methods showed similar performance to Q-learning using random forests and augmented inverse probability weighted estimation, and when the sample size was large they were close to $\ell_1$ penalized least squares and residual weighted learning.
For Scenario 2, $\mathcal{V}_{1}=1.34$, $\mathcal{V}_{2}=1.55$, and $\mathcal{V}^{\ast}=1.95$. The optimal regime $d^{\ast}({x})$ is 1 if $0.5x_4^2-0.3x_3<0.4$, and 2 otherwise. The decision boundary was nonlinear. $\ell_1$ penalized least squares and residual weighted learning with linear kernel both failed due to model misspecification. The adaptive causal $k$-nearest neighbor method yielded the best performance.
The causal $k$-nearest neighbor method showed similar performance to Q-learning using random forests, residual weighted learning with Gaussian kernel and augmented inverse probability weighted estimation.
For Scenario 3, $\mathcal{V}_{1}=1.88$, $\mathcal{V}_{2}=1.94$, and $\mathcal{V}^{\ast}=2.37$. The optimal regime $d^{\ast}({x})$ is 1 if $x^2_3+x^2_4<1$, and 2 otherwise. The decision boundary was highly nonlinear.
Similar to Scenario 2, our proposed adaptive causal $k$-nearest neighbor approach outperformed all other methods. The causal $k$-nearest neighbor method yielded similar performance to other nonlinear methods including Q-learning using random forests, residual weighted learning with Gaussian kernel and augmented inverse probability weighted estimation.


We move now to the moderate dimensional cases ($p=25$). The simulation results are shown in Table~\ref{tab:valueknnhigh}. In Scenario 1, $\ell_1$ penalized least squares outperformed other methods because of correct model specification and inside variable selection techniques. Residual weighted learning methods yielded similar performance to $\ell_1$ penalized least squares due to their variable selection mechanism. The causal $k$-nearest neighbor regime was not comparable with others in this scenario because of the lack of a variable selection procedure. It is well known that nearest neighbor rules deteriorate when there are irrelevant covariates present in the data. The proposed adaptive causal $k$-nearest neighbor approach showed similar performance to Q-learning using random forests and augmented inverse probability weighted estimation, and when the sample size was large it was close to $\ell_1$ penalized least squares and residual weighted learning methods. Their good performance can be explained by variable selection. The adaptive causal $k$-nearest neighbor approach carries out variable selection through the adaptive metric selection. Q-learning using random forests and augmented inverse probability weighted estimation, as two tree methods, have a built-in mechanism to perform variable selection \citep{Breiman:CART84}. In Scenarios 2 and 3, $\ell_1$ penalized least squares and residual weighted learning with linear kernel failed due to misspecification; again, causal $k$-nearest neighbor failed due to the lack of variable selection. Four nonparametric methods with variable selection, Q-learning using random forests, augmented inverse probability weighted estimate, residual weighted learning with Gaussian kernel and the adaptive causal $k$-nearest neighbor approach, stood out. Among them, our proposed adaptive causal $k$-nearest neighbor method ranked the first in both scenarios.

\begin{table}
\centering
\caption{Mean (standard deviation) of empirical value functions evaluated on on the test set for Scenarios 1-3 when the dimension is moderate ($p=25$). The best value function for each scenario and sample size combination is in bold.}
\label{tab:valueknnhigh}
\begin{threeparttable}
\begin{tabular}{lccccc}
\addlinespace
\toprule
& $n=50$ & $n=100$ & $n=200$ & $n=400$ & $n=800$ \\
\midrule
\multicolumn{6}{c}{Scenario 1 (Optimal value $2.09$)} \\
\midrule
$\ell_1$-PLS & \textbf{1.91 (0.12)} & \textbf{2.00 (0.07)} & \textbf{2.04 (0.04)} & \textbf{2.06 (0.02)} & \textbf{2.08 (0.02)} \\
Q-RF & 1.83 (0.10) & 1.91 (0.08) & 1.97 (0.06) & 2.01 (0.03) & 2.04 (0.01) \\
RWL-VS-Linear & 1.84 (0.12) & 1.97 (0.08) & 2.03 (0.05) & \textbf{2.06 (0.03)} & \textbf{2.08 (0.01)} \\
RWL-VS-Gaussian & 1.82 (0.13) & 1.92 (0.10) & 2.02 (0.07) & \textbf{2.06 (0.04)} & 2.07 (0.03)  \\
AIPWE & 1.80 (0.15) & 1.90 (0.12) & 1.97 (0.08) & 2.01 (0.05) & 2.03 (0.03) \\
CNN & 1.79 (0.09) & 1.82 (0.08) & 1.85 (0.06) & 1.89 (0.05) & 1.91 (0.04) \\
ACNN & 1.77 (0.12) & 1.83 (0.13) & 1.91 (0.12) & 2.00 (0.07) & 2.03 (0.04) \\
\midrule
\multicolumn{6}{c}{Scenario 2 (Optimal value $1.95$)} \\
\midrule
$\ell_1$-PLS & 1.44(0.06) & 1.44(0.06) & 1.44(0.07) & 1.45(0.06) & 1.46(0.06) \\
Q-RF &  1.48 (0.09) & 1.54 (0.09) & 1.68 (0.08) & 1.81 (0.06) & 1.87 (0.03)  \\
RWL-VS-Linear & 1.49 (0.07) & 1.52 (0.07) & 1.57 (0.07) & 1.62 (0.05) & 1.65 (0.04) \\
RWL-VS-Gaussian &  1.51 (0.09) & 1.61 (0.13) & \textbf{1.77 (0.12)} & \textbf{1.87 (0.07)} & \textbf{1.91 (0.04)} \\
AIPWE & 1.48 (0.09) & 1.55 (0.12) & 1.71 (0.13) & 1.82 (0.08) & 1.88 (0.04) \\
CNN & 1.49 (0.06) & 1.52 (0.06) & 1.56 (0.05) & 1.60 (0.05) & 1.65 (0.04) \\
ACNN &  \textbf{1.52 (0.11)} & \textbf{1.62 (0.15)} & {1.76 (0.13)} & {1.86 (0.07)} & {1.90 (0.04)} \\
\midrule
\multicolumn{6}{c}{Scenario 3 (Optimal value $2.37$)} \\
\midrule
$\ell_1$-PLS & 1.89 (0.02) & 1.89 (0.02) & 1.89 (0.02) & 1.89 (0.02) & 1.89 (0.02) \\
Q-RF & 1.92 (0.03) & 1.94 (0.04) & 1.99 (0.05) & 2.07 (0.06) & 2.18 (0.05)  \\
RWL-VS-Linear & 1.90(0.03) & 1.90(0.03) & 1.91(0.04) & 1.92(0.04) & 1.93(0.05) \\
RWL-VS-Gaussian & 1.94 (0.07) & 2.02 (0.13) & 2.20 (0.12) & \textbf{2.30 (0.07)} & {2.32 (0.06)}  \\
AIPWE & 1.92 (0.06) & 2.00 (0.12) & 2.15 (0.11) & 2.24 (0.04) & 2.27 (0.03)  \\
CNN & 1.92 (0.01) & 1.93 (0.02) & 1.94 (0.02) & 1.96 (0.02) & 1.98 (0.02) \\
ACNN & \textbf{1.99 (0.11)} & \textbf{2.10 (0.12)} & \textbf{2.23 (0.09)} & \textbf{2.30 (0.04)} & \textbf{2.33 (0.02)}  \\
\bottomrule
\end{tabular}
\begin{tablenotes}
\item $\ell_1$-PLS, $\ell_1$ penalized least squares; Q-RF, Q-learning using random forests; RWL-VS-Linear, residual weighted learning with variable selection and linear kernel; RWL-VS-Gaussian, residual weighted learning with variable selection and Gaussian kernel; AIPWE, augmented inverse probability weighted estimation; CNN, causal $k$-nearest neighbor; ACNN, adaptive causal $k$-nearest neighbor.
\end{tablenotes}
\end{threeparttable}
\end{table}

Here, the covariates were independent. We also run simulations to assess performance of our proposed methods when the covariates were correlated. The results are similar to the independent cases presented above. Details are collected in Appendix D.

We then evaluated the performance of our proposed methods on data with more than two treatment arms (say, $L=3$). The simulation setup was similar to that with two treatment arms. We generated $p$-dimensional vectors of clinical covariates as before. The treatment $A$ was generated from $\mathcal{A}=\{1, 2, 3\}$ independently of ${X}$ with  $\pi_1({x})=\pi_2({x})=\pi_3({x})=1/3$ for any ${x}\in\mathbb{R}^{p}$. The response $R$ was normally distributed with mean $Q_0({x},a)$ and standard deviation 1. We considered two scenarios with different choices of $Q_0({x},a)$:
\begin{enumerate}
\item[(4)] $Q_0({x},1)=(1+0.5x_1+0.8x_2+x_3-0.5x_4+0.7x_5)-0.5x_3$;

$Q_0({x},2)=(1+0.5x_1+0.8x_2+x_3-0.5x_4+0.7x_5) + 0.2x_3$;

$Q_0({x},3)=(1+0.5x_1+0.8x_2+x_3-0.5x_4+0.7x_5) + 0.5x_4$.
\item[(5)] $Q_0({x},1)=(0.5x_1+0.8x_2+0.3{x}_3-0.5{x}_4+0.7{x}_5)+(1.6\tilde{x}_3+0.4x_4+0.2)$;

$Q_0({x},2)=(0.5x_1+0.8x_2+0.3{x}_3-0.5{x}_4+0.7{x}_5)+(0.4{x}_3+2\tilde{x}_4-0.2)$;

$Q_0({x},3)=(0.5x_1+0.8x_2+0.3{x}_3-0.5{x}_4+0.7{x}_5)+(0.4{x}_3+0.4x_4+1)$,

where $\tilde{x}_j = \min(x_j^2,1)$, for $j=3,4$.
\end{enumerate}

We compared the performance of the following four methods: (1) $\ell_1$ penalized least squares proposed by \citet{Qian:ITR2011}; (2) Q-learning using random forests as described in \citet{Taylor2015:QRF}; (3) the proposed causal $k$-nearest neighbor method; and (4) the proposed adaptive causal $k$-nearest neighbor method. Residual weighted learning and augmented inverse probability weighted estimation methods have only been implemented for two treatment arms, and so are not included here. For each scenario, we varied sample sizes from $n=150$, $300$, $600$, to $1200$, and repeated the simulation 500 times. The independent test set was with a sample size of 30,000.

The simulation results on the low dimensional cases ($p=5$) are presented in Table~\ref{tab:valueknn3low}.
For Scenario 4, $\mathcal{V}_{1}=\mathcal{V}_{2}=\mathcal{V}_{3}=1.65$, and $\mathcal{V}^{\ast}=2.04$. The decision boundary is linear. $\ell_1$ penalized least squares produced the best performance because of correct model specification. The other nonparametric methods, Q-learning using random forests, causal $k$-nearest neighbor and adaptive causal $k$-nearest neighbor methods, showed similar performance.
For Scenario 5, $\mathcal{V}_{1}=1.67$, $\mathcal{V}_{2}=1.48$, $\mathcal{V}_{3}=1.65$, and $\mathcal{V}^{\ast}=2.21$. The decision boundary is nonlinear. $\ell_1$ penalized least squares was not comparable with other nonparametric methods as the postulated model was misspecified. Our proposed adaptive causal $k$-nearest neighbor method produced the best performance. The causal $k$-nearest neighbor method showed similar performance to Q-learning using random forests.

\begin{table}[t]
\centering
\caption{Mean (standard deviation) of empirical value functions evaluated on the test set for Scenarios 4 and 5 when the dimension is low ($p=5$). The best value function for each scenario and sample size combination is in bold.}
\label{tab:valueknn3low}
\begin{threeparttable}
\begin{tabular}{lcccc}
\addlinespace
\toprule
& $n=150$ & $n=300$ & $n=600$ & $n=1200$ \\
\midrule
\multicolumn{5}{c}{Scenario 4 (Optimal value $2.04$)} \\
\midrule
$\ell_1$-PLS & \textbf{1.94(0.08)} & \textbf{1.99(0.06)} & \textbf{2.02(0.02)} & \textbf{2.03(0.02)} \\
Q-RF & 1.82 (0.06) & 1.87 (0.05) & 1.91 (0.03) & 1.95 (0.02) \\
CNN & 1.83 (0.08) & 1.89 (0.06) & 1.93 (0.04) & 1.96 (0.03) \\
ACNN & 1.81 (0.09) & 1.87 (0.09) & 1.91 (0.18) & 1.97 (0.04) \\
\midrule
\multicolumn{5}{c}{Scenario 5 (Optimal value $2.21$)} \\
\midrule
$\ell_1$-PLS & 1.82(0.12) & 1.90(0.11) & 1.98(0.08) & 2.02(0.05) \\
Q-RF &  1.88 (0.07) & 1.97 (0.05) & 2.04 (0.03) & 2.09 (0.02)  \\
CNN & 1.85 (0.09) & 1.94 (0.06) & 2.02 (0.04) & 2.07 (0.02) \\
ACNN & \textbf{1.89 (0.10)} & \textbf{1.99 (0.08)} & \textbf{2.07 (0.06)} & \textbf{2.12 (0.03)} \\
\bottomrule
\end{tabular}
\begin{tablenotes}
\item $\ell_1$-PLS, $\ell_1$ penalized least squares; Q-RF, Q-learning using random forests;
CNN, causal $k$-nearest neighbor; ACNN, adaptive causal $k$-nearest neighbor.
\end{tablenotes}
\end{threeparttable}
\end{table}

We then increased the dimensionality to 25. The results are presented in Table~\ref{tab:valueknn3high}. Again, $\ell_1$ penalized least squares produced the best performance in Scenario 4, and the adaptive causal $k$-nearest neighbor method in Scenario 5. The causal $k$-nearest neighbor approach was not comparable with others in both scenarios due to the curse of dimensionality.

\begin{table}[t]
\centering
\caption{Mean (standard deviation) of empirical value functions evaluated on the test set for Scenarios 4 and 5 when the dimension is moderate ($p=25$). The best value function for each scenario and sample size combination is in bold.}
\label{tab:valueknn3high}
\begin{threeparttable}
\begin{tabular}{lcccc}
\addlinespace
\toprule
& $n=150$ & $n=300$ & $n=600$ & $n=1200$ \\
\midrule
\multicolumn{5}{c}{Scenario 4 (Optimal value $2.04$)} \\
\midrule
$\ell_1$-PLS &  \textbf{1.85(0.09)} & \textbf{1.94(0.07)} & \textbf{2.00(0.04)} & \textbf{2.02(0.02)} \\
Q-RF & 1.77 (0.06) & 1.83 (0.06) & 1.90 (0.04) & 1.95 (0.03) \\
CNN & 1.72 (0.04) & 1.76 (0.04) & 1.78 (0.04) & 1.82 (0.04) \\
ACNN & 1.72 (0.08) & 1.77 (0.09) & 1.83 (0.09) & 1.91 (0.07) \\
\midrule
\multicolumn{5}{c}{Scenario 5 (Optimal value $2.21$)} \\
\midrule
$\ell_1$-PLS & 1.72(0.08) & 1.79(0.09) & 1.90(0.08) & 1.99(0.05) \\
Q-RF & 1.72 (0.08) & 1.82 (0.08) & 1.92 (0.05) & 1.99 (0.04) \\
CNN & 1.66 (0.04) & 1.69 (0.03) & 1.73 (0.03) & 1.76 (0.03) \\
ACNN & \textbf{1.74 (0.12)} & \textbf{1.87 (0.13)} & \textbf{2.01 (0.10)} & \textbf{2.11 (0.04)} \\
\bottomrule
\end{tabular}
\begin{tablenotes}
\item $\ell_1$-PLS, $\ell_1$ penalized least squares; Q-RF, Q-learning using random forests;
CNN, causal $k$-nearest neighbor; ACNN, adaptive causal $k$-nearest neighbor.
\end{tablenotes}
\end{threeparttable}
\end{table}

When the dimension is low, the causal $k$-nearest neighbor regime produced comparable performance to other alternatives. The adaptive selection on the distance metric enhances the causal $k$-nearest neighbor regime. From the simulations, the adaptive causal $k$-nearest neighbor method showed at least similar results to the causal $k$-nearest neighbor regime. As we explained before, when the tuning parameter $\Delta$ is very small, the adaptive causal $k$-nearest neighbor approach is almost equivalent to the causal $k$-nearest neighbor approach. Considering the superior performance of the adaptive causal $k$-nearest neighbor over the causal $k$-nearest neighbor approach, especially when the dimensionality is large, we suggest the adaptive causal $k$-nearest neighbor method for general practical use.

\section{Data analysis}
We applied the proposed methods to analyze data from a chronic depression clinical trial \citep{Keller:depression}.
Patients with non-psychotic chronic major depressive disorder were randomized in a 1:1:1 ratio to either Nefazodone, cognitive behavioral-analysis system of psychotherapy, or the combination of two therapies.
The primary outcome measurement in efficacy was the score on the 24-item Hamilton rating scale for depression. Lower score is desirable. We considered 50 pre-treatment covariates as in \citet{Zhao:OWL2012}. We excluded some patients with missing covariate values from the analyses. The data used here consisted of 647 patients. Among them, 216, 220, and 211 patients were assigned to three arms, respectively. Each clinical covariate was scaled to $[-1,+1]$, as described in \citet{Hsu2003:svmguide}.

Since the trial had three treatment arms, we compared the performance of the adaptive causal $k$-nearest neighbor regime with $\ell_1$ penalized least squares and Q-learning using random forests. Residual weighted learning and augmented inverse probability weighted estimation methods can only deal with two treatments. From the simulation studies, the adaptive regime outperformed the causal $k$-nearest neighbor one especially when the dimension of covariates was large, so we only considered the adaptive regime in this section. Outcomes used in the analyses were opposites of the scores on the 24-item Hamilton rating scale for depression.
We used a nested 10-fold cross-validation procedure for an unbiased comparison \citep{Ambroise:bias}. 
To obtain reliable estimates, we repeated the nested cross-validation procedure 100 times with different fold partitions.

The mean value functions over 100 repeats and the standard deviations are presented in Table~\ref{tab:depressionknn}.
The adaptive causal $k$-nearest neighbor regime achieved a similar performance to $\ell_1$ penalized least squares and Q-learning using random forests. All methods assigned the combination treatment to almost every patient. The original analysis in \citet{Keller:depression} indicated that the combination treatment is significantly more efficacious than either treatment alone. Our analysis confirmed that this is indeed true.

We also performed pairwise comparisons between two treatment arms. We included two residual weighted learning methods with variable selection and augmented inverse probability weighted estimation in the analysis.
The analysis results are also presented in Table~\ref{tab:depressionknn}. 
For comparison between Nefazodone and cognitive behavioral-analysis system of psychotherapy, the adaptive causal $k$-nearest neighbor regime was slightly better than other methods except for residual weighted learning with linear kernel. For comparison between Nefazodone and combination therapy, all methods produced similar performance. For comparison between cognitive behavioral-analysis system of psychotherapy and combination therapy, the adaptive causal $k$-neareast neighbor regime did not perform comparably to other methods.
We carried out the significance test described in Section \ref{sec:adaptive} to compare the regimes by the adaptive causal $k$-nearest neighbor and residual weighted learning with linear kernel, and the difference between them was not statistically significant.

\begin{table}
\centering
\caption{Mean score (standard deviation) on Hamilton rating scale for depression from the cross-validation procedure using different methods. Lower score is better.}
\label{tab:depressionknn}
\begin{threeparttable}
\begin{tabular}{lcccc}
\addlinespace
\toprule
& NFZ vs CBASP & & &\\
& vs COMB & NFZ vs CBASP & NFZ vs COMB & CBASP vs COMB \\
\midrule
$\ell_1$-PLS & 11.19 (0.15) & 16.30 (0.39) & 11.20 (0.16) & 10.95 (0.09) \\
Q-RF & 11.11 (0.13) & 16.27 (0.44) & 11.05 (0.18) & 10.93 (0.09) \\
RWL-VS-Linear & $-$ & 15.45 (0.37) & 11.09 (0.29) & 10.88 (0.05) \\
RWL-VS-Gaussian & $-$ & 16.29 (0.44) & 11.33 (0.25) & 11.07 (0.28) \\
AIPWE & $-$ & 16.45 (0.41) & 10.97 (0.15) & 10.96 (0.14) \\
ACNN & 11.18 (0.27) & 15.70 (0.39) & 11.03 (0.27) & 11.41 (0.28) \\
\bottomrule
\end{tabular}
\begin{tablenotes}
\item $\ell_1$-PLS, $\ell_1$ penalized least squares \citep{Qian:ITR2011}; Q-RF, Q-learning using random forests \citep{Taylor2015:QRF}; RWL-VS-Linear, residual weighted learning with variable selection and linear kernel \citep{Zhou2015:RWL}; RWL-VS-Gaussian, residual weighted learning with variable selection and Gaussian kernel \citep{Zhou2015:RWL}; AIPWE, augmented inverse probability weighted estimation \citep{Zhang2012:ClassificationITR}; ACNN, adaptive causal $k$-nearest neighbor.
NFZ, Nefazodone; CBASP, cognitive behavioral-analysis system of psychotherapy; COMB, combination of Nefazodone and cognitive behavioral-analysis system of psychotherapy.
\end{tablenotes}
\end{threeparttable}
\end{table}

The adaptive causal $k$-nearest neighbor regime showed a statistically equivalent performance to other methods on the chronic depression clinical trial data.

\section{Discussion}
In this article, we have proposed a simple causal $k$-nearest neighbor method to optimal treatment regimes, and developed an adaptive method to determine the distance metric. As shown in the simulation and data studies, the adaptive method can rival and improve upon more sophisticated methods, especially when the decision boundary is nonlinear.

Variable selection plays a critical role in identifying the optimal treatment regime when the dimension of covariates is large, as shown in the simulation studies. $\ell_1$ penalized least squares methods use the least absolute shrinkage and selection operator for variable selection. Residual weighted learning performs variable selection through the elastic-net penalty for linear kernels and through covariate-scaling for Gaussian kernels \citep{Zhou2015:RWL}.
As a tree method, augmented inverse probability weighted estimation is equipped with a built-in variable selection mechanism \citep{Breiman:CART84}. Our proposed adaptive causal $k$-nearest neighbor method applies an adaptive distance metric to perform variable selection. Recently, several researchers highlighted the importance of variable selection for optimal treatment regimes \citep{Gunter2011:Qualitative,Zhou2015:RWL}. Variable selection in optimal treatment regimes has its own characteristics. There are two different types of covariates related to outcomes $R$, predictive and prescriptive covariates.
Predictive covariates are useful to the prediction of outcomes; and prescriptive covariates are used to prescribe optimal treatment regimes \citep{Gunter2011:Qualitative}. \citet{Athey2016:Partitioning} discussed several ways of splitting on prescriptive covariates rather than predictive covariates on causal trees.
The variable selection in the adaptive causal $k$-nearest neighbor regime is to identify prescriptive covariates through tuning with the additional parameter $\Delta$.
As pointed out by an anonymous reviewer, in practice, the number of predictive covariates may be much larger than the number of prescriptive covariates. So it is important and challenging to carry out variable selection for optimal treatment regimes. 

The causal $k$-nearest neighbor methods are simple and fast; they possess nice theoretical properties; as nonparametric methods, they are free of model specification; they naturally work with multiple-arm trials; the variable selection in the adaptive causal $k$-nearest neighbor regime identifies prescriptive covariates to further improve finite sample performance.

\section*{Acknowledgement}
This work was sponsored by the National Cancer Institute. We are grateful to the
editors and the reviewers for their insightful comments, which have led to important improvements in this paper.

\appendix
\makeatletter   
 \renewcommand{\@seccntformat}[1]{APPENDIX~{\csname the#1\endcsname}.\hspace*{1em}}
 \makeatother

\vspace{40pt}
\noindent \textbf{\LARGE APPENDIX}

\vspace{10pt}
\noindent We prove Theorems \ref{thm:knnconsistency} and \ref{thm:knnconvergence} of the main paper in Appendix A and B. The proofs are based on theoretical results for nearest neighbor rules in regression. For completeness, we collect the theorems and lemmas needed in the proofs in Appendix C. We present additional simulation results in Appendix D.

\section{Proof of Theorem \ref{thm:knnconsistency}}
The following lemma shows that consistency of $\hat{m}_\ell({x})$, $\ell=1,\cdots,L$, guarantees consistency of the rule $d^{CNN}$.
\begin{lemma} \label{thm:knnvaluebound}
The causal $k$-nearest neighbor rule in \eqref{eq:knnregimerule} of the main paper satisfies the following bound for any distribution $P$ for $({X},A,R)$,
\begin{equation*}
\mathcal{V}(d^*) - \mathcal{V}(d^{CNN}) \leq \sum_{\ell=1}^L\int|\hat{m}_\ell({x})-m_\ell({x})|\mu(d{x}).
\end{equation*}
\end{lemma}
\begin{proof}[Proof of Lemma \ref{thm:knnvaluebound}]
Note that the value function of any rule $d$,
\begin{equation*}
\mathcal{V}(d) := \bE\big(R^{\ast}(d(X))\big)=\sum_{\ell=1}^L\bE\big(R^{\ast}(\ell)\bI(d(X)=\ell)\big).
\end{equation*}
Thus, by fixing ${x}\in\mathcal{X}$, we have
\begin{eqnarray*}
&&\sum_{\ell=1}^L\bE\Big(R^{\ast}(\ell)\bI(d^{\ast}(X)=\ell)\big|{X}={x}\Big)
-
\sum_{\ell=1}^L\bE\Big(R^{\ast}(\ell)\bI(d^{CNN}(X)=\ell)\big|{X}={x}\Big)\\
&=&\sum_{\ell=1}^Lm_\ell({x})\left(\mathbb{I}(d^*({x})=\ell)-\mathbb{I}(d^{CNN}({x})=\ell)\right) \\
&=&m_{\ell_1}({x}) - m_{\ell_2}({x}),
\end{eqnarray*}
where $\ell_1=d^*({x})$ and $\ell_2=d^{CNN}({x})$, and the expectation $\mathbb{E}$ is with respect to $P$ for $({X},A,R^{\ast}(\ell),\ell=1,\cdots,L)$. By the construction of $d^{CNN}({x})$, we have
\begin{eqnarray*}
& & m_{\ell_1}({x}) - m_{\ell_2}({x})\\
&\leq& \left(m_{\ell_1}({x}) - \hat{m}_{\ell_1}({x})\right) - \left(m_{\ell_2}({x}) - \hat{m}_{\ell_2}({x})\right)\\
&\leq&\sum_{\ell=1}^L\left|m_{\ell}({x}) - \hat{m}_{\ell}({x})\right|.
\end{eqnarray*}
The desired result follows by taking expectation over ${X}$ on both sides.
\end{proof}

Now it is sufficient to prove, for any $\ell\in\{1,\cdots,L\}$,
\begin{equation*}
\int|\hat{m}_{\ell}({x})-m_{\ell}({x})|\mu(d{x})\rightarrow 0
\end{equation*}
in probability or almost surely, as $n\rightarrow\infty$. We start from a simpler $k$-nearest neighbor rule, for $\ell\in\{1,\cdots,L\}$,
\begin{equation} \label{eq:simplerknn}
\hat{m}'_\ell({x}) = \sum_{i=1}^kR_{(i,n)}({x})\frac{\bI(A_{(i,n)}({x})=\ell)}{k\pi_\ell\left({X}_{(i,n)}({x})\right)}.
\end{equation}
The relationship between $\hat{m}_\ell({x})$ and $\hat{m}'_\ell({x})$ is that
\begin{equation*}
\hat{m}_\ell({x}) = \frac{\hat{m}'_\ell({x})}{\frac{1}{k}{\sum_{i=1}^k\frac{\bI(A_{(i,n)}({x})=\ell)}{\pi_\ell\left({X}_{(i,n)}({x})\right)}}}.
\end{equation*}
By the law of large numbers, the denominator
\begin{equation*}
\frac{1}{k}{\sum_{i=1}^k\frac{\bI(A_{(i,n)}({x})=\ell)}{\pi_\ell\left({X}_{(i,n)}({x})\right)}}\rightarrow 1 \quad \textrm{a.s.}
\end{equation*}
as $k\rightarrow\infty$. Thus it is now sufficient to prove, for any $\ell\in\{1,\cdots,L\}$,
\begin{equation*}
\int|\hat{m}'_{\ell}({x})-m_{\ell}({x})|\mu(d{x})\rightarrow 0
\end{equation*}
in probability or almost surely, as $n\rightarrow\infty$.

From now on, we use $m_{\ell}({x})=\bE(R|X=x,A=\ell)$. For weak consistency, we will prove a slightly stronger result,
$
\bE\left(\int|\hat{m}'_{\ell}({x})-m_{\ell}({x})|\mu(d{x})\right) \rightarrow 0.
$
We rewrite $\hat{m}'_{\ell}({x})$ as
\begin{equation*}
\hat{m}'_\ell({x}) = \sum_{i=1}^n V_{n,i}^\ell({x}) R_{(i,n)}({x}),
\end{equation*}
where the weights are
\begin{equation*}
V_{n,i}^\ell({x})=\left\{\begin{array}{cc}
\frac{\bI(A_{(i,n)}({x})=\ell)}{k\pi_\ell\left({X}_{(i,n)}({x})\right)}, & {\rm if\ } i\leq k,\\
0 & {\rm if\ } i>k.
\end{array}\right.
\end{equation*}
Note that the $V_{n,i}^\ell({x})$'s depend on $X_1$, $\cdots$, $X_n$, $A_1$, $\cdots$, $A_n$. For the $k$-nearest neighbor regression, the weights depend on $X_1$, $\cdots$, $X_n$. Thus the theoretical results in $k$-nearest neighbor regression may not apply to our settings for optimal treatment regimes.

We proceed by checking a couple of conditions as in Stone's Theorem in Appendix \ref{sec:stone}, and then prove the weak consistency for the optimal treatment regime settings.
\begin{enumerate}
\item[(i)] There is a constant $c$ such that for every nonnegative measurable
function $f$ satisfying $\mathbb{E}f({X})<\infty$ and any $n$,
\begin{equation*}
\mathbb{E}\left\{\sum_{i=1}^n|V^{\ell}_{n,i}({X})|f({X}_i)\right\}\leq c\mathbb{E}f({X}).
\end{equation*}
\begin{proof}[Proof:]
\begin{equation*}
\mathbb{E}\left\{\sum_{i=1}^nV_{n,i}^\ell({X})f({X}_i)\right\} \leq  \frac{1}{k\zeta}\mathbb{E}\left\{\sum_{i=1}^k f({X}_{(i,n)}({X}))\right\}
\leq \frac{\gamma_d}{\zeta}\mathbb{E}(f({X})).
\end{equation*}
The last inequality is due to Lemma \ref{thm:neighborexpectation} in Appendix \ref{sec:stone}.
\end{proof}
\item[(ii)] For all $\delta>0$,
\begin{equation*}
\lim_{n\rightarrow\infty}\mathbb{E}\left\{\sum_{i=1}^n|V^\ell_{n,i}({X})|\mathbb{I}(||{X}_i-{X}||>\delta)\right\}=0.
\end{equation*}
\begin{proof}[Proof:]
\begin{eqnarray*}
&& \mathbb{E}\left\{\sum_{i=1}^n|V_{n,i}^\ell({X})|\mathbb{I}(||{X}_i-{X}||>a)\right\}\\
&= & \int\bE\left\{\sum_{i=1}^n|V_{n,i}^\ell({x})|\mathbb{I}(||{X}_i-{x}||>a)\right\}\mu(d{x})\\
&\leq &
\int\bE\left\{\frac{1}{k\zeta}\sum_{i=1}^k\mathbb{I}(||{X}_{(i,n)}({x})-{x}||>a)\right\}\mu(d{x})\\
&\leq &\frac{1}{\zeta}
\int P(||{X}_{(k,n)}({x})-{x}||>a)\mu(d{x}).
\end{eqnarray*}
For ${x}\in support(\mu)$, when $k/n\rightarrow0$, Lemma \ref{thm:neighborhood} in Appendix \ref{sec:stone} implies $P(||{X}_{(k,n)}({x})-{x}||>a)\rightarrow0$. Then the dominated convergence theorem implies condition (ii).
\end{proof}
\end{enumerate}
Now we are ready to prove $\bE\left(\int|\hat{m}'_{\ell}({x})-m_{\ell}({x})|\mu(d{x})\right) \rightarrow 0$. Fixing $x\in\mathcal{X}$, we have
\begin{eqnarray*}
&&|\hat{m}'_{\ell}({x})-m_{\ell}({x})|=\left|\sum_{i=1}^n V_{n,i}^\ell({x}) R_{(i,n)}({x})-m_{\ell}({x})\right|\\
&\leq& \left|\sum_{i=1}^n V_{n,i}^\ell({x})\Big(R_{(i,n)}({x})-m_{\ell}\big(X_{(i,n)}({x})\big)\Big)\right|+\left|\sum_{i=1}^n V_{n,i}^\ell({x})\Big(m_{\ell}\big(X_{(i,n)}({x})\big)-m_{\ell}({x})\Big)\right| \\
& &+\left|\left(\sum_{i=1}^n V_{n,i}^\ell({x})-1\right)m_{\ell}({x})\right|\\
&=&I_{n1}(x)+I_{n2}(x)+I_{n3}(x).
\end{eqnarray*}
Note that $\sum_{i=1}^n V_{n,i}^\ell({x})\rightarrow1$ almost surely and $\sum_{i=1}^n V_{n,i}^\ell({x})\leq1/\zeta$. Then $\bE(\int I_{n3}(x)\mu(d{x}))\rightarrow 0$ by the dominated convergence theorem.

For the first term $I_{n1}$,
\begin{eqnarray*}
&&\bE\int\left|\sum_{i=1}^n V_{n,i}^\ell({x})\Big(R_{(i,n)}({x})-m_{\ell}\big(X_{(i,n)}({x})\big)\Big)\right|\mu(d{x})\\
&\leq& \frac{1}{\zeta}\bE\int\left|\frac{1}{k}\sum_{i=1}^k \bI(A_{(i,n)}({x})=\ell)\Big(R_{(i,n)}({X})-m_{\ell}\big(X_{(i,n)}({X})\big)\Big)\right|\mu(d{x}).
\end{eqnarray*}
By the law of large numbers, as $k\rightarrow\infty$, $\frac{1}{k}\sum_{i=1}^k \bI(A_{(i,n)}({x})=\ell)\Big(R_{(i,n)}({X})-m_{\ell}\big(X_{(i,n)}({X})\big)\Big)\rightarrow0$ almost surely. Then by the dominated convergence theorem, $\bE(\int I_{n1}(x)\mu(d{x}))\rightarrow 0$.

Because of Theorem A.1 in \citet[page 589]{Gyorfi2002:Nonparametric}, for $\epsilon>0$, we can choose $m'_{\ell}(x)$ bounded and uniformly
continuous such that $\int|m'_{\ell}({x})-m_{\ell}({x})|\mu(dx)<\epsilon$.  For the second term $I_{n2}$, we have
\begin{eqnarray*}
I_{n2}(x)&\leq & \left|\sum_{i=1}^n V_{n,i}^\ell({x})\Big(m_{\ell}\big(X_{(i,n)}({x})\big)-m'_{\ell}\big(X_{(i,n)}({x})\big)\Big)\right| + \left|\sum_{i=1}^n V_{n,i}^\ell({x})\Big(m'_{\ell}\big(X_{(i,n)}({x})\big)-m'_{\ell}({x})\Big)\right|\\
&& + \left|\sum_{i=1}^n V_{n,i}^\ell({x})\Big(m'_{\ell}({x})-m_{\ell}({x})\Big)\right| = J_{n1}(x)+J_{n2}(x)+J_{n3}(x).
\end{eqnarray*}
By the construction of $m'_{\ell}(x)$, for $J_{n3}(x)$, we have,
\begin{equation*}
\bE\int J_{n3}(x)\mu(dx) \leq \frac{1}{\zeta}\int\Big(m'_{\ell}({x})-m_{\ell}({x})\Big)\mu(dx)\leq \frac{\epsilon}{\zeta}.
\end{equation*}
For the term $J_{n1}(x)$, by condition (ii), we have
\begin{equation*}
\bE\int J_{n1}(x)\mu(dx) \leq c \int\Big(m'_{\ell}({x})-m_{\ell}({x})\Big)\mu(dx)\leq c\epsilon.
\end{equation*}

Because of uniform continuity of $m'_{\ell}(x)$, we can find a $\delta$ such that $|m'_{\ell}(x)-m'_{\ell}(y)|\leq\epsilon$ for any $x$ and $y\in\mathcal{X}$ satisfying $||x-y||\leq\delta$. For the term $J_{n2}(x)$, we have
\begin{eqnarray*}
\bE\int J_{n2}(x)\mu(dx) &\leq& \bE\left|\sum_{i=1}^n V_{n,i}^\ell({X})\Big(m'_{\ell}\big(X_{(i,n)}({X})\big)-m'_{\ell}({X})\Big)\bI(|X_{(i,n)}({X})-X|>\delta)\right|\\
&&+\bE\int\left|\sum_{i=1}^n V_{n,i}^\ell({x})\Big(m'_{\ell}\big(X_{(i,n)}({x})\big)-m'_{\ell}({x})\Big)\bI(|X_{(i,n)}({x})-x|\leq\delta)\right|\mu(dx) \\
&\leq& 2\sup_{x\in\mathcal{X}}\Big(m'_{\ell}(x)\Big)\bE\left|\sum_{i=1}^n V_{n,i}^\ell({X})\bI(|X_{(i,n)}({X})-X|>\delta)\right|+\frac{\epsilon}{\zeta}.
\end{eqnarray*}
By condition (ii), we have,
\begin{equation*}
\limsup_{n\rightarrow\infty}\bE\int J_{n2}(x)\mu(dx)\leq \frac{\epsilon}{\zeta}.
\end{equation*}
So combining the terms for $J_{n1}(x)$, $J_{n2}(x)$ and $J_{n3}(x)$, when $\epsilon\rightarrow0$, we have
$\bE(\int I_{n2}(x)\mu(d{x}))\rightarrow 0$. Now we finish the proof for $\bE\left(\int|\hat{m}'_{\ell}({x})-m_{\ell}({x})|\mu(d{x})\right) \rightarrow 0$. The weak consistency of
the causal $k$-nearest neighbor regime (3) in the main paper follows by Lemma \ref{thm:knnvaluebound}.

We next show strong consistency for bounded outcomes. By the Borel-Cantelli Lemma, it suffices to show the following theorem. The proof follows an idea in  \citet[Chapter 11]{Devroye:PatternRecog1996}.
\begin{theorem} \label{thm:strongbound}
For any distribution $P$ for $({X},A,R)$ satisfying assumptions (A1), (A3), and $|R|\leq M<\infty$ for some constant $M$, if $k\rightarrow\infty$ and $k/n\rightarrow0$, then for every $\epsilon>0$ there exists an $n_0(\epsilon)$ such that for $n \geq n_0$
\begin{equation*}
P\left(\int|\hat{m}'_{\ell}({x})-m_{\ell}({x})|d({x})>\epsilon\right)\leq 2\exp(-cn\epsilon^2),
\end{equation*}
where $c>0$ depends only on the dimension $p$, $M$ and $\zeta$.
\end{theorem}

REMARK: The inequality in Theorem \ref{thm:strongbound} does not imply a $\sqrt{n}$-consistent rate since it is only valid when $n\geq n_0$, where $n_0$ depends on $\epsilon$.
\begin{proof}[Proof of Theorem \ref{thm:strongbound}:]
Fix ${x}\in\mathcal{X}$.
Denote $\rho_n({x})=||{x}-{X}_{(k,n)}({x})||$. Also define $\rho^*_n({x})$ as the solution of the equation ${k}/{n}=\mu(S_{{x},\rho^*_n({x})})$. Since distance ties occur with probability zero in $\mu$, the solution always exists. Now define the rule
\begin{equation*}
\hat{m}^*_\ell({x}) = \sum_{i=1}^nR_i\frac{\bI(A_i=\ell)}{k\pi_\ell({X}_i)}\bI\left(||{x}-{X}_i||\leq\rho^*_n({x})\right),
\end{equation*}
and consider the following decomposition,
\begin{equation*}
|\hat{m}'_{\ell}({x})-m_{\ell}({x})|\leq|\hat{m}'_{\ell}({x})-\hat{m}^*_{\ell}({x})|+|\hat{m}^*_{\ell}({x})-m_{\ell}({x})|.
\end{equation*}
For the first term on the right-hand side, we obtain,
\begin{eqnarray*}
& & |\hat{m}'_{\ell}({x})-\hat{m}^*_{\ell}({x})| \\
& = & \frac{1}{k}\Big|\sum_{i=1}^nR_i\frac{\bI(A_i=\ell)}{\pi_\ell({X}_i)}\bI\left(||{x}-{X}_i||\leq\rho_n({x})\right)-\sum_{i=1}^nR_i\frac{\bI(A_i=\ell)}{\pi_\ell({X}_i)}\bI\left(||{x}-{X}_i||\leq\rho^*_n({x})\right)\Big|\\
&=&
\frac{1}{k}\Big|\sum_{i=1}^nR_i\frac{\bI(A_i=\ell)}{\pi_\ell({X}_i)}\Big(\bI\left(||{x}-{X}_i||\leq\rho_n({x})\right)-\bI\left(||{x}-{X}_i||\leq\rho^*_n({x})\right)\Big)\Big|\\
&\leq& \frac{1}{k}\sum_{i=1}^n\Big|R_i\frac{\bI(A_i=\ell)}{\pi_\ell({X}_i)}\Big(\bI\left(||{x}-{X}_i||\leq\rho_n({x})\right)-\bI\left(||{x}-{X}_i||\leq\rho^*_n({x})\right)\Big)\Big|\\
& \leq & \frac{M}{k\zeta}\sum_{i=1}^n\Big|\bI\left(||{x}-{X}_i||\leq\rho_n({x})\right)-\bI\left(||{x}-{X}_i||\leq\rho^*_n({x})\right)\Big|\\
& = & \frac{M}{\zeta}\Big|\frac{1}{k}\sum_{i=1}^n\bI\left(||{x}-{X}_i||\leq\rho^*_n({x})\right)-1\Big|
=\frac{M}{\zeta}\Big|\frac{1}{k}\sum_{i=1}^n\bI\left({X}_i\in S_{{x},\rho^*_n({x})}\right)-1\Big|.
\end{eqnarray*}
Denote $\hat{s}({x})=\frac{1}{k}\sum_{i=1}^n\bI\left({X}_i\in S_{{x},\rho^*_n({x})}\right)$.
Thus,
\begin{equation} \label{eq:boundm}
|\hat{m}'_{\ell}({x})-m_{\ell}({x})|\leq\frac{M}{\zeta}|\hat{s}({x})-1|+|\hat{m}^*_{\ell}({x})-m_{\ell}({x})|.
\end{equation}

Observe that $\bE(\hat{s}({x}))=1$, then we have,
\begin{eqnarray*}
\bE\left\{\int|\hat{s}({x})-1|\mu(d{x})\right\}
& \leq & \int\sqrt{\bE\left\{\big(\hat{s}({x})-1\big)^2\right\}}\mu(d{x}) \\
&=&\int\sqrt{\frac{n}{k^2}\textrm{Var}\left(I({X}\in S_{{x},\rho^*_n({x})})\right)}\mu(d{x}) \leq \frac{1}{\sqrt{k}}.
\end{eqnarray*}
Thus we obtain,
\begin{eqnarray*}
\lim_{n\rightarrow\infty}\bE\left(\int|\hat{s}({x})-1|\mu(d{x})\right) & = & 0,\\
\textrm{and } \quad \lim_{n\rightarrow\infty}\bE\left(\int|\hat{m}'_{\ell}({x})-\hat{m}^*_{\ell}({x})|\mu(d{x})\right) & = & 0.
\end{eqnarray*}
We already showed that
\begin{equation*}
\lim_{n\rightarrow\infty}\bE\left(\int|\hat{m}'_{\ell}({x})-m_{\ell}({x})|\mu(d{x})\right) = 0.
\end{equation*}
So we have,
\begin{equation*}
\lim_{n\rightarrow\infty}\bE\left(\int|\hat{m}^*_{\ell}({x})-m_{\ell}({x})|\mu(d{x})\right) = 0.
\end{equation*}
Fix $\epsilon>0$. Then we can find an $n_0$ such that, for $n\geq n_0$,
\begin{eqnarray*}
\bE\left(\int|\hat{s}({x})-1|\mu(d{x})\right) & < & \frac{\zeta}{8M}\epsilon,\\
\textrm{and } \quad \bE\left(\int|\hat{m}^*_{\ell}({x})-m_{\ell}({x})|\mu(d{x})\right) & < & \frac{\epsilon}{8}.
\end{eqnarray*}
Then, by (\ref{eq:boundm}), we have, when $n\geq n_0$,
\begin{eqnarray} \label{eq:probboundm}
&&P(\int|\hat{m}'_{\ell}({x})-m_{\ell}({x})|\mu(d{x})>\epsilon) \\ \nonumber &\leq&P\left(\int|\hat{s}({x})-1|\mu(d{x})-\bE\int|\hat{s}({x})-1|\mu(d{x})>\frac{\zeta}{4M}\epsilon\right)\\
\nonumber &&+P\left(\int|\hat{m}^*_{\ell}({x})-m_{\ell}({x})|\mu(d{x})-\bE\int|\hat{m}^*_{\ell}({x})-m_{\ell}({x})|\mu(d{x})>\frac{1}{2}\epsilon\right).
\end{eqnarray}
We will use McDiarmid's inequality \citep[Theorem 9.2]{Devroye:PatternRecog1996} to bound each term on the right-hand side of (\ref{eq:probboundm}). Fix an arbitrary realization of the data $({x}_j,a_j,r_j)_{j=1}^n$. Replace $({x}_i,a_i,r_i)$ by $({x}'_i,a'_i,r'_i)$, changing the value of $\hat{m}^*_{\ell}({x})$ to $\hat{m}^*_{\ell,i}({x})$. Thus
\begin{equation*}
\Big|\int|\hat{m}^*_{\ell}({x})-m_{\ell}({x})|\mu(d{x})-\int|\hat{m}^*_{\ell,i}({x})-m_{\ell}({x})|\mu(d{x})\Big|
\leq
\int|\hat{m}^*_{\ell}({x})-\hat{m}^*_{\ell,i}({x})|\mu(d{x}).
\end{equation*}
And
\begin{equation*}
|\hat{m}^*_{\ell}({x})-\hat{m}^*_{\ell,i}({x})| =
\frac{1}{k}\Big|r_i\frac{\bI(a_i=\ell)}{\pi_\ell({x}_i)}\bI\left(||{x}-{x}_i||\leq\rho^*_n({x})\right)-r'_i\frac{\bI(a'_i=\ell)}{\pi_\ell({x}'_i)}\bI\left(||{x}-{x}'_i||\leq\rho^*_n({x})\right)\Big|
\end{equation*}
is bounded by $2M/(k\zeta)$, and can differ from zero only if $||{x}-{x}_i||\leq\rho^*_n({x})$ or $||{x}-{x}'_i||\leq\rho^*_n({x})$. Note that $||{x}-{x}_i||\leq\rho^*_n({x})$ if and only if $\mu(S_{{x},||{x}-{x}_i||})\leq k/n$. By Lemma \ref{thm:neighborcovering}, the measure of such ${x}$ is bounded by $\gamma_pk/n$. Thus by McDiarmid's inequality,
\begin{equation*}
P\left(\int|\hat{m}^*_{\ell}({x})-m_{\ell}({x})|\mu(d{x})-\bE\int|\hat{m}^*_{\ell}({x})-m_{\ell}({x})|\mu(d{x})>\frac{1}{2}\epsilon\right)\leq \exp\left(-\frac{n\epsilon^2\zeta^2}{32M^2\gamma_p^2}\right).
\end{equation*}
Similarly,
\begin{equation*}
\Big|\int|\hat{s}({x})-1|\mu(d{x})-\int|\hat{s}_i({x})-1|\mu(d{x})\Big|
\leq
\int|\hat{s}({x})-\hat{s}_i({x})|\mu(d{x}),
\end{equation*}
and
\begin{equation*}
|\hat{s}({x})-\hat{s}_i({x})|=\frac{1}{k}
\Big|\bI\left(||{x}-{x}_i||\leq\rho^*_n({x})\right)-\bI\left(||{x}-{x}'_i||\leq\rho^*_n({x})\right)\Big|
\end{equation*}
is bounded by $1/k$. By McDiarmid's inequality again,
\begin{equation*}
P\left(\int|\hat{s}({x})-1|\mu(d{x})-\bE\int|\hat{s}({x})-1|\mu(d{x})>\frac{\zeta}{4M}\epsilon\right)
\leq \exp\left(-\frac{n\epsilon^2\zeta^2}{32M^2\gamma_p^2}\right).
\end{equation*}
The desired result follows from (\ref{eq:probboundm}) with $\displaystyle c=\frac{\zeta^2}{32M^2\gamma_p^2}$. 
\end{proof}
Now we prove (ii), strong consistency for unbounded $R$. A counterpart of Lemma 5 in \citet{Devroye1994:StrongConsistency} is needed for the setting of optimal treatment regimes. The proof follows the idea in \citet{Gyorfi1991:unbounded}.
\begin{lemma} \label{thm:unboundedR}
Consider the $k$-nearest neighbor estimate $\hat{m}'_{\ell}({x})$ in (\ref{eq:simplerknn}). Then
\begin{equation*}
\int|\hat{m}'_{\ell}({x})-m_{\ell}({x})|\mu(d{x})\rightarrow 0
\end{equation*}
almost surely for all distributions of $({X},A,R)$ satisfying assumptions (A1)$\sim$(A3) if the following two conditions are satisfied:
\begin{itemize}
\item[(a)] $\int|\hat{m}'_{\ell}({x})-m_{\ell}({x})|\mu(d{x})\rightarrow 0$ almost surely for all distributions of $({X},A,R)$ satisfying assumptions (A1) and (A3) with bounded $R$.
\item[(b)] There exists a constant $c>0$ such that, for all distributions of $({X},A,R)$ satisfying assumptions (A2) and (A3),
\begin{equation} \label{eq:limsupbound}
\limsup_{n\rightarrow\infty}\frac{1}{k}\sum_{i=1}^k\int|R_{(i,n)}({x})|\mu(d{x})\leq c\bE|R|  \quad a.s.
\end{equation}
\end{itemize}
\end{lemma}
\begin{proof}[Proof of Lemma \ref{thm:unboundedR}:]
For an arbitrary $M$, let
\begin{equation*}
T_i = \left\{
\begin{array}{ll}
R_i & {\rm if}\: |R_i|\leq M, \\
M\textrm{sign}(R_i) & {\rm otherwise},\\
\end{array}
\right.
\end{equation*}
for $i=1,\cdots,n$. $T$ is defined similarly. Let
$\hat{t}_\ell({x})$ be the functions $\hat{m}'_\ell({x})$, respectively, when $R_i$ is replaced by $T_i$, for $i=1,\cdots,n$. Denote $t_\ell({x})=\bE(T|{X}={x},A=\ell)$ for $\ell=1,\cdots,L$. Now,
\begin{eqnarray*}
&&\limsup_{n\rightarrow\infty}\int|\hat{m}'_{\ell}({x})-m_{\ell}({x})|\mu(d{x})\\ &\leq&\limsup_{n\rightarrow\infty}\int|\hat{m}'_{\ell}({x})-\hat{t}_{\ell}({x})|\mu(d{x})
+\limsup_{n\rightarrow\infty}\int|\hat{t}_{\ell}({x})-{t}_{\ell}({x})|\mu(d{x})
+\int|{t}_{\ell}({x})-{m}_{\ell}({x})|\mu(d{x}).
\end{eqnarray*}
For the first term on the right-hand side, we have,
\begin{eqnarray*}
&&\limsup_{n\rightarrow\infty}\int|\hat{m}'_{\ell}({x})-\hat{t}_{\ell}({x})|\mu(d{x})\\
&\leq&\limsup_{n\rightarrow\infty}\frac{1}{k}\sum_{i=1}^k\int|R_{(i,n)}({x})-T_{(i,n)}({x})|\frac{\bI(A_{(i,n)}({x})=\ell)}{\pi_\ell\left({X}_{(i,n)}({x})\right)}\mu(d{x}) \\ &\leq&\limsup_{n\rightarrow\infty}\frac{1}{k\zeta}\sum_{i=1}^k\int|R_{(i,n)}({x})-T_{(i,n)}({x})|\mu(d{x}) \\&\leq&\frac{c}{\zeta}\bE|R-T| \quad a.s.
\end{eqnarray*}
The last inequality is due to condition (b) since  $\bE|R-T|<\infty$. The second term converges almost surely to zero by condition (a). By Jensen's inequality, the third term satisfies,
\begin{equation*}
\int|{t}_{\ell}({x})-{m}_{\ell}({x})|\mu(d{x})
=\bE\Big|\bE(R-T|X,A=\ell)\Big|
\leq \bE\left(|R-T|\Big|A=\ell\right)\leq \frac{1}{\zeta}\bE|R-T|.
\end{equation*}
Thus we have,
\begin{equation*}
\limsup_{n\rightarrow\infty}\int|\hat{m}'_{\ell}({x})-\hat{t}_{\ell}({x})|\mu(d{x})\leq \frac{c+1}{\zeta}\bE|R-T| \quad a.s.
\end{equation*}
By the dominated convergence theorem, $\bE|R-T|\rightarrow0$ as $M\rightarrow\infty$. The desired result now follows as $M\rightarrow\infty$. 
\end{proof}

For strong consistency in (ii), since we have already proved strong consistency for bounded $R$, it is enough to prove (\ref{eq:limsupbound}).

We need some geometric properties of the nearest neighborhood. Define a cone $C({x},{s})$ to be the collection of all ${x}'\in\mathbb{R}^p$ for which either ${x}'={x}$ or $\textrm{angle}({x}'-{x},{s})\leq \pi/6$.
Let $S$ be a minimal subset of $\mathbb{R}^p$ such that a collection of cones $C({x},{s})$ for ${s}\in S$ covers $\mathbb{R}^p$. By Lemma \ref{thm:covering1}, such an $S$ exists, and its cardinality $|S|$ is $\gamma_p$.
Let $D_i$ be the collection of all ${x}\in\mathbb{R}^p$ such that ${X}_i$ is one of its $k$ nearest neighbors.
Define the sets $C_{i,{s}}=C({X}_i,{s})$ for $i=1,\cdots,n$ and ${s}\in S$. Let $B_{i,{s}}$ be the subset of $C_{i,{s}}$ consisting of all ${x}$ that are among the $k$ nearest neighbors of ${X}_i$ in the set $\{{X}_1,\cdots,{X}_{i-1},{X}_{i+1},\cdots,{X}_n, {x}\}\bigcap C_{i,{s}}$. If the number of ${X}_j$'s ($j\neq i$) contained in $C_{i,{s}}$ is fewer than $k$, then $B_{i,{s}}=C_{i,{s}}$.

Observe that, by Lemma~\ref{thm:neighborhoodA} and Lemma~\ref{thm:neighborhoodB},
\begin{equation*}
\limsup_{n\rightarrow\infty}\frac{n}{k}\max_{i}\mu(D_i)\leq\limsup_{n\rightarrow\infty}\frac{n}{k}\max_{i}\sum_{{s}\in S}\mu(B_{i,{s}})\leq\sum_{{s}\in S}\limsup_{n\rightarrow\infty}\frac{n}{k}\max_{i}\mu(B_{i,{s}})\leq 2\gamma_p.
\end{equation*}
Then, we have,
\begin{eqnarray*}
&&\limsup_{n\rightarrow\infty}\frac{1}{k}\sum_{i=1}^k\int|R_{(i,n)}({x})|\mu(d{x})\\
&=&\limsup_{n\rightarrow\infty}\frac{1}{k}\sum_{i=1}^n|R_i|\mu(D_i)\\
&\leq&\limsup_{n\rightarrow\infty}\left(\frac{1}{n}\sum_{i=1}^n|R_i|\right)\limsup_{n\rightarrow\infty}\left(\frac{n}{k}\max_i\mu(D_i)\right)\\
&\leq&2\gamma_p\limsup_{n\rightarrow\infty}\left(\frac{1}{n}\sum_{i=1}^n|R_i|\right) = 2\gamma_p\bE|R| \quad a.s.
\end{eqnarray*}
Thus strong consistency in (ii) follows from Lemma \ref{thm:unboundedR} and Lemma \ref{thm:knnvaluebound}. The proof of Theorem \ref{thm:knnconsistency} is complete.

REMARK: In practical use, we prefer Stone's estimate in (3) of the main paper to break distance ties. Consider a simpler rule,
\begin{equation*}
\tilde{m}'_\ell({x}) = \frac{1}{k}\sum_{i\in A_k({x})}R_{(i,n)}({x})\frac{\bI(A_{(i,n)}({x})=\ell)}{\pi_\ell\left({X}_{(i,n)}({x})\right)}+\frac{k-|A_k({x})|}{k|B_k({x})|}\sum_{i\in B_k({x})}R_{(i,n)}({x})\frac{\bI(A_{(i,n)}({x})=\ell)}{\pi_\ell\left({X}_{(i,n)}({x})\right)}.
\end{equation*}
When $k\rightarrow\infty$, $\tilde{m}'_\ell$ is asymptotically equivalent to Stone's estimate $\tilde{m}_\ell$ in (3) of the main paper.
Assumption (A3) has a connotation of breaking distance ties randomly as demonstrated in the main paper.
If the assumption does not hold, a small uniform variable $U\sim uniform[0,\epsilon]$ independent of $({X},A,R)$ may be added to the vector ${X}$. We may perform the causal $k$-nearest neighbor rule on $({X},U)$. By Jensen's inequality,
\begin{equation*}
\bE\int_{0}^{\epsilon}\int|\hat{m}'_{\ell}({x},u)-m_{\ell}({x})|\mu(d{x})du\\
\geq \bE\int\left|\bE\left(\int_{0}^{\epsilon}\hat{m}'_{\ell}({x},u)du\Big|D_n\right)-m_{\ell}({x})\right|\mu(d{x}).
\end{equation*}
Fixing the data $D_n=\{({X}_i, A_i, R_i): i=1,\cdots,n\}$, we can always find a small enough $\epsilon$ such that
\begin{equation*}
\bE\left(\int_{0}^{\epsilon}\hat{m}'_{\ell}({x},u)du\Big|D_n\right) = \tilde{m}'_\ell({x}).
\end{equation*}
Thus $\tilde{m}'_\ell$ is better than $\hat{m}'_\ell$ on $({X},U)$, and then Stone's tie-breaking rule $\tilde{m}_\ell$ in (3) of the main paper is asymptotically better than random tie-breaking.

\section{Proof of Theorem \ref{thm:knnconvergence}}
By Lemma~\ref{thm:knnvaluebound},
\begin{equation*}
\bE\left\{\left(\mathcal{V}(d^*) - \mathcal{V}(d^{NN})\right)^2\right\} \leq L\sum_{\ell=1}^L\bE\left\{\left(\int|\hat{m}_\ell({x})-m_\ell({x})|\mu(d{x})\right)^2\right\}.
\end{equation*}
So it suffices to show the following theorem for the bound on $\bE\left\{(\int|\hat{m}_\ell({x})-m_\ell({x})|\mu(d{x}))^2\right\}$,
for $\ell=1,\cdots,L$.

\begin{theorem} \label{thm:knnconvergencemj}
For any distribution $P$ for $({X},A,R)$ satisfying assumptions (A1$'$)$\sim$(A4$'$), and $\ell=1,\cdots,L$,
\begin{enumerate}
\item[(i)] If $p=1$,
\begin{equation}\label{eq:knnconvergencemj1}
\bE\left\{\left(\int|\hat{m}_\ell({x})-m_\ell({x})|\mu(d{x})\right)^2\right\}\leq c^2\sigma^2\frac{1}{k}+16c\rho^2C^2\frac{k}{n}.
\end{equation}
\item[(ii)] If $p=2$,
\begin{equation} \label{eq:knnconvergencemj2}
\bE\left\{\left(\int|\hat{m}_\ell({x})-m_\ell({x})|\mu(d{x})\right)^2\right\}\leq c^2\sigma^2\frac{1}{k}+8c\rho^2C^2\frac{k}{n}\left(1+\log\left(\frac{n}{k}\right)\right).
\end{equation}
\item[(iii)] If $p\geq3$,
\begin{equation} \label{eq:knnconvergencemj3}
\bE\left\{\left(\int|\hat{m}_\ell({x})-m_\ell({x})|\mu(d{x})\right)^2\right\}\leq c^2\sigma^2\frac{1}{k}+\frac{8c\rho^2C^2}{1-2/p}\left\lfloor\frac{n}{k}\right\rfloor^{-\frac{2}{d}}.
\end{equation}
\end{enumerate}
\end{theorem}
\begin{proof}[Proof of Theorem~\ref{thm:knnconvergencemj}:]
Let
\begin{equation*}
{m}^*_\ell({x}) = \bE\left(\hat{m}_\ell({x})|{X}_1,A_1,\cdots,{X}_n,A_n\right)=\sum_{i=1}^kW_{n,i}^\ell({x})m_\ell\left({X}_{(i,n)}({x})\right).
\end{equation*}
The last equality is due to the fact that $W_{n,i}^\ell({x})=0$ if $A_{(i,n)}({x})\neq\ell$. We have the decomposition
\begin{eqnarray*}
& &\bE\left\{\left(\int|\hat{m}_\ell({x})-m_\ell({x})|\mu(d{x})\right)^2\right\} \leq \bE\int\big(\hat{m}_\ell({x})-m_\ell({x})\big)^2\mu(d{x}) \\
&=& \bE\int\big(\hat{m}_\ell({x})-m^*_\ell({x})\big)^2\mu(d{x})+
\bE\int\big(m^*_\ell({x})-m_\ell({x})\big)^2\mu(d{x}).
\end{eqnarray*}
For the first term on the right-hand side,
\begin{eqnarray*}
& &\bE\int\big(\hat{m}_\ell({x})-m^*_\ell({x})\big)^2\mu(d{x})\\
&=& \bE\int\left(\sum_{i=1}^kW_{n,i}^\ell({x})\Big(R_{(i,n)}({x})-m_\ell\big({X}_{(i,n)}({x})\big)\Big)\right)^2\mu(d{x})\\
&=& \bE\int\sum_{i=1}^k\Big(W_{n,i}^\ell({x})\Big)^2\Big(R_{(i,n)}({x})-m_\ell\big({X}_{(i,n)}({x})\big)\Big)^2\mu(d{x})\\
&=& \bE\int\sum_{i=1}^k\Big(W_{n,i}^\ell({x})\Big)^2\sigma^2_\ell\big({X}_{(i,n)}({x})\big)\mu(d{x})\\
&\leq& c^2\sigma^2\frac{1}{k}.
\end{eqnarray*}
For the second term,
\begin{eqnarray*}
& &\bE\int\big(m^*_\ell({x})-m_\ell({x})\big)^2\mu(d{x}) \\
&=&\bE\int\left(\sum_{i=1}^kW_{n,i}^\ell({x})\Big(m_\ell\big({X}_{(i,n)}({x})\big)-m_\ell({x})\Big)\right)^2\mu(d{x})\\
&\leq&\bE\int\sum_{i=1}^k\Big(W_{n,i}^\ell({x})\Big)^2\sum_{i=1}^k\Big(m_\ell\big({X}_{(i,n)}({x})\big)-m_\ell({x})\Big)^2\mu(d{x})\\
&\leq&cC^2\bE||{X}_{(k,n)}({X})-{X}||^2.
\end{eqnarray*}
The desired results in Theorem \ref{thm:knnconvergencemj} now follow directly from Lemma \ref{thm:neighborbound}.
\end{proof}
When $p=1$, take $k\propto n^{1/2}$, and the right-hand side of (\ref{eq:knnconvergencemj1}) is $O(n^{-1/2})$. When $p=2$, take $k\propto n^{1/2-\epsilon}$ for any $\epsilon>0$, and the right-hand side of (\ref{eq:knnconvergencemj2}) is $O(n^{-1/2+\epsilon})$. When $\epsilon$ is very small, its rate of convergence will be arbitrarily close to $1/2$. When $p\geq3$, take $k\propto n^{2/(p+2)}$, and the right-hand side of (\ref{eq:knnconvergencemj3}) is $O(n^{-2/(p+2)})$. Theorem \ref{thm:knnconvergence} is now proved.

\section{Background on $k$-nearest neighbor regression} \label{sec:stone}
The setup in this section is for regression analysis, and is different from the setup in the main paper. In regression analysis one considers a random vector $({X},Y)$, where ${X}$ is $\mathbb{R}^p$-valued, and $Y$ is $\mathbb{R}$-valued. Let $D_n$ be the set of observed data defined by
\begin{equation*}
D_n = \{({X}_1, Y_1),\cdots,({X}_n, Y_n)\},
\end{equation*}
where $({X}_1,Y_1)$, $\cdots$, $({X}_n,Y_n)$ and $({X},Y)$ are independent and identically distributed (i.i.d.) random variables. Let $m({x})=\mathbb{E}(Y|{X}={x})$. In the regression problem one wants to use the data $D_n$ in order to construct an estimate $\hat{m}: \mathbb{R}^p\rightarrow \mathbb{R}$ of the regression function $m$. Here $\hat{m}({x})=\hat{m}({x},D_n)$ is a measurable function of ${x}$ and the data. We first state Stone's Theorem \citep{Stone1977:NN}. The theorem was applied to prove consistency of kernel and nearest neighbor estimates in the literature \citep{Devroye:PatternRecog1996,Gyorfi2002:Nonparametric}. The theorem considers a regression function estimate taking the form
\begin{equation*}
\hat{m}({x}) = \sum_{i=1}^nW_{n,i}({x})Y_i,
\end{equation*}
where the weights $W_{n,i}({x}) = W_{n,i}({x},{X}_1,\cdots,{X}_n) \in \mathbb{R}$ depend on ${X}_1,\cdots,{X}_n$.

\begin{theorem} [Stone's Theorem]
Assume that the following conditions are satisfied for any distribution of ${X}$:
\begin{enumerate}
\item[(i)] There is a constant $c$ such that for every nonnegative measurable
function $f$ satisfying $\mathbb{E}f({X})<\infty$ and any $n$,
\begin{equation*}
\mathbb{E}\left\{\sum_{i=1}^n|W_{n,i}({X})|f({X}_i)\right\}\leq c\mathbb{E}f({X}).
\end{equation*}
\item[(ii)] There is a $D\geq1$ such that
\begin{equation*}
\mathbb{P}\left\{\sum_{i=1}^n|W_{n,i}({X})|\leq D\right\} = 1,
\end{equation*}
for all $n$.
\item[(iii)] For all $a>0$,
\begin{equation*}
\lim_{n\rightarrow\infty}\mathbb{E}\left\{\sum_{i=1}^n|W_{n,i}({X})|\mathbb{I}(||{X}_i-{X}||>a)\right\}=0.
\end{equation*}
\item[(iv)]
\begin{equation*}
\sum_{i=1}^nW_{n,i}({X})\rightarrow 1 \quad \textrm {in probability.}
\end{equation*}
\item[(v)]
\begin{equation*}
\max_i|W_{n,i}({X})| \rightarrow 0 \quad \textrm {in probability.}
\end{equation*}
\end{enumerate}
Then the corresponding regression function estimate $\hat{m}$ converges in mean to $m$, \textit{i.e.},
\begin{equation*}
\mathbb{E}\left(\int|\hat{m}({x})-m({x})|\mu(d{x})\right) \rightarrow 0
\end{equation*}
for all distributions of $({X},Y)$ with $\mathbb{E}|Y|<\infty$.
\end{theorem}

We fix $x\in\mathbb{R}^p$, and reorder the observed data $({X}_1,Y_1)$, $\cdots$, $({X}_n,Y_n)$ according to increasing values of $||{X}_i-{x}||$. The reordered data sequence is denoted by
\begin{equation*}
({X}_{(1,n)}({x}), Y_{(1,n)}({x})),\cdots,({X}_{(n,n)}({x}), Y_{(n,n)}({x})).
\end{equation*}
Thus ${X}_{(k,n)}({x})$ is the $k$th nearest neighbor of ${x}$.

We introduce some results on the nearest neighborhood of ${x}$ which are useful in proving theorems in the main paper. Denote the probability measure for ${X}$ by $\mu$.
In this section, we assume that distance ties occur with probability zero in $\mu$.
Let $S_{{x},\epsilon}$ be the closed ball centered at ${x}$ of radius $\epsilon>0$. Define $support(\mu)=\{{x}: \textrm{for all } \epsilon>0, \mu(S_{{x},\epsilon})>0 \}$.

\begin{lemma}[Lemma 5.1 in \citet{Devroye:PatternRecog1996}] \label{thm:neighborhood}
If ${x}\in support(\mu)$ and $\lim_{n\rightarrow\infty}k/n=0$, then $||{X}_{(k,n)}({x})-{x}||\rightarrow 0$ with probability one.
\end{lemma}

Let us define the cone $C({x},{s})$ to be the collection of all ${x}'\in\mathbb{R}^p$ for which either ${x}'={x}$ or $\textrm{angle}({x}'-{x},{s})\leq \pi/6$. The following lemma shows that a finite set of such cones covers $\mathbb{R}^p$.
\begin{lemma}[Lemma 5.5 in \citet{Devroye:PatternRecog1996}] \label{thm:covering1}
There exists a finite set $S\subset\mathbb{R}^p$ such that
\begin{equation*}
\mathbb{R}^p = \bigcup_{{s}\in S}C({x},{s}),
\end{equation*}
regardless of how ${x}\in\mathbb{R}^p$ is picked.
Furthermore, define $\gamma_p$ as the minimal number of elements in $S$. Then $\gamma_p$ depends only on the dimension $p$, and
\begin{equation*}
\gamma_p \leq \left(1+2\sqrt{2-\sqrt{3}}\right)^p-1.
\end{equation*}
\end{lemma}

The next several lemmas will enable us to establish weak and strong consistency of nearest neighbor rules.
\begin{lemma}[Lemma 11.1 in \citet{Devroye:PatternRecog1996}] \label{thm:neighborcovering}
Let $B_a({x}')=\{{x}:\mu(S_{{x},||{x}-{x}'||})\leq a\}$. Then for all ${x}'\in\mathbb{R}^p$,
\begin{equation*}
\mu(B_a({x}'))\leq\gamma_pa.
\end{equation*}
\end{lemma}

\begin{lemma}[Lemma 5.3 in \citet{Devroye:PatternRecog1996}] \label{thm:neighborexpectation}
For any integrable function $f$, any $n$, and any $k\leq n$,
\begin{equation*}
\sum_{i=1}^k\mathbb{E}\left(|f({X}_{(i,n)}({X}))|\right)\leq k\gamma_p\mathbb{E}(|f({X})|).
\end{equation*}
\end{lemma}

Let $A_i$ be the collection of all ${x}\in\mathbb{R}^p$ such that ${X}_i$ is one of its $k$ nearest neighbors. Let $S$ be a minimal subset of $\mathbb{R}^p$, such that a collection of cones $C({x},{s})$ for ${s}\in S$ covers $\mathbb{R}^p$. Thus $\gamma_p=|S|$, the cardinality of this set. Define the sets $C_{i,{s}}=C({X}_i,{s})$. Let $B_{i,{s}}$ be the subset of $C_{i,{s}}$ consisting of all ${x}$ that are among the $k$ nearest neighbors of ${X}_i$ in the set $\{{X}_1,\cdots,{X}_{i-1},{X}_{i+1},\cdots,{X}_n, {x}\}\bigcap C_{i,{s}}$.

\begin{lemma}[Lemma 6 in \citet{Devroye1994:StrongConsistency}] \label{thm:neighborhoodA}
If ${x}\in A_i$, then ${x}\in\bigcup_{{s}\in S}B_{i,{s}}$, and thus
\begin{equation*}
\mu(A_i)\leq\sum_{{s}\in S}\mu(B_{i,{s}}).
\end{equation*}
\end{lemma}

\begin{lemma}[Lemma 8 in \citet{Devroye1994:StrongConsistency}] \label{thm:neighborhoodB}
If $k/\log(n)\rightarrow\infty$ and $k/n\rightarrow0$, then
\begin{equation*}
\limsup_{n\rightarrow\infty}\frac{n}{k}\max_i\mu(B_{i,{s}})\leq 2 \qquad a.s.
\end{equation*}
\end{lemma}

\citet{Devroye1994:StrongConsistency} applied an additional independent random variable to break distance ties. It is easy to translate the proofs of the previous two lemmas to the case where distance ties occur with probability zero in $\mu$, so we can skip the proofs. The next lemma is helpful to show the rate of convergence in the main paper.

\begin{lemma}[Corollary 6 in \citet{Biau2010:BaggedNN}] \label{thm:neighborbound}
Suppose that $\mu$ has a compact support with diameter $2\rho$. Then
\begin{enumerate}
\item[(i)] If $p=1$,
\begin{equation*}
\bE||{X}_{(i,n)}-{X}||^2\leq\frac{16\rho^2i}{n}.
\end{equation*}
\item[(ii)] If $p=2$,
\begin{equation*}
\bE||{X}_{(i,n)}-{X}||^2\leq\frac{8\rho^2i}{n}\left(1+\log\left(\frac{n}{i}\right)\right).
\end{equation*}
\item[(iii)] If $p\geq 3$,
\begin{equation*}
\bE||{X}_{(i,n)}-{X}||^2\leq\frac{8\rho^2\lfloor n/i\rfloor^{-\frac{2}{p}}}{1-2/p}.
\end{equation*}
\end{enumerate}
\end{lemma}

\section{Additional simulations}
The simulation studies in the main paper considered independent covariates. We also run simulations to assess performance of our proposed methods when the covariates were correlated. The simulation setup was almost the same as the previous setup. We generated $p$-dimensional vectors of clinical covariates. The first two covariates were independent bernoulli random variables with success probability of $0.5$, and the remaining covariates were from a multivariate normal distribution. The mean of each normal covariate is 0, and the variance is 1. The covariance of two different normal covariates is 0.5. In this set of simulations, we did not include residual weighted learning methods in the comparison due to their high computational costs.

The simulation results on the low dimensional data ($p=5$) are presented in Table~\ref{tab:corrvalueknnlow}. For Scenario 1, $\mathcal{V}_{1}=1.85$, $\mathcal{V}_{2}=1.45$, and $\mathcal{V}^{\ast}=2.09$. The decision boundary was a linear combination of a binary covariate and a continuous covariate. $\ell_1$ penalized least squares performed very well since its model was correctly specified. Our proposed causal $k$-nearest neighbor and adaptive causal $k$-nearest neighbor methods showed similar performance to Q-learning using random forests and augmented inverse probability weighted estimation, and when the sample size was large they were close to $\ell_1$ penalized least squares.
For Scenario 2, $\mathcal{V}_{1}=1.31$, $\mathcal{V}_{2}=1.57$, and $\mathcal{V}^{\ast}=1.95$. The decision boundary was nonlinear.
$\ell_1$ penalized least squares failed due to model misspecification. The adaptive causal $k$-nearest neighbor approach yielded the best performance. The causal $k$-nearest neighbor approach showed similar performance to Q-learning using random forests and augmented inverse probability weighted estimation.
For Scenario 3, $\mathcal{V}_{1}=1.88$, $\mathcal{V}_{2}=1.94$, and $\mathcal{V}^{\ast}=2.41$. The decision boundary was highly nonlinear. Again, $\ell_1$ penalized least squares failed due to model misspecification, and the adaptive causal $k$-nearest neighbor approach outperformed all other methods.
The simulation results for the moderate dimensional case ($p=25$) are shown in Table~\ref{tab:corrvalueknnhigh}. In Scenario 1, $\ell$ penalized least squares again outperformed other methods because of correct model specification and inside variable selection techniques. The proposed causal $k$-nearest neighbor and adaptive causal $k$-nearest neighbor methods showed similar performance to Q-learning using random forests and augmented inverse probability weighted estimation. In Scenarios 2 and 3, $\ell_1$ penalized least squares failed due to misspecification. The causal $k$-nearest neighbor was slightly worse than Q-learning using random forests and augmented inverse probability weighted estimation when the sample size was large. Our proposed adaptive causal $k$-nearest neighbor approach outperformed other methods.

\begin{table}
\centering
\caption{Mean (standard deviation) of empirical value functions evaluated on the test set for Scenarios 1-3 when the dimension is low ($p=5$) and covariates 3-5 are correlated. The best value function for each scenario and sample size combination is in bold.}
\label{tab:corrvalueknnlow}
\begin{threeparttable}
\begin{tabular}{lccccc}
\addlinespace
\toprule
& $n=50$ & $n=100$ & $n=200$ & $n=400$ & $n=800$ \\
\midrule
\multicolumn{6}{c}{Scenario 1 (Optimal value $2.09$)} \\
\midrule
$\ell_1$-PLS & \textbf{1.99 (0.08)} & \textbf{2.02 (0.07)} & \textbf{2.05 (0.04)} & \textbf{2.07 (0.02)} & \textbf{2.07 (0.01)} \\
Q-RF & 1.93 (0.09) & 1.98 (0.05) & 2.01 (0.03) & 2.03 (0.02) & 2.04 (0.01) \\
AIPWE & 1.93 (0.12) & 1.97 (0.09) & 2.00 (0.06) & 2.02 (0.04) & 2.04 (0.03) \\
CNN & 1.93 (0.11) & 1.98 (0.07) & 2.01 (0.05) & 2.04 (0.03) & 2.05 (0.02)  \\
ACNN & 1.91 (0.13) & 1.96 (0.11) & 2.00 (0.08) & 2.04 (0.04) & 2.06 (0.02) \\
\midrule
\multicolumn{6}{c}{Scenario 2 (Optimal value $1.95$)} \\
\midrule
$\ell_1$-PLS & 1.53 (0.11) & 1.57 (0.10) & 1.61 (0.08) & 1.63 (0.04) & 1.65 (0.03) \\
Q-RF & \textbf{1.66 (0.10)} & 1.75 (0.07) & 1.83 (0.04) & 1.87 (0.03) & 1.90 (0.02)  \\
AIPWE & 1.65 (0.15) & 1.76 (0.13) & 1.84 (0.08) & 1.88 (0.04) & 1.90 (0.03) \\
CNN &  1.65 (0.12) & 1.75 (0.08) & 1.82 (0.05) & 1.87 (0.03) & 1.90 (0.02)  \\
ACNN &  \textbf{1.66 (0.15)} & \textbf{1.77 (0.11)} & \textbf{1.85 (0.08)} & \textbf{1.89 (0.04)} & \textbf{1.91 (0.03)} \\
\midrule
\multicolumn{6}{c}{Scenario 3 (Optimal value $2.41$)} \\
\midrule
$\ell_1$-PLS & 1.88 (0.05) & 1.88 (0.04) & 1.88 (0.04) & 1.88 (0.04) & 1.88 (0.04) \\
Q-RF & 2.11 (0.09) & 2.21 (0.07) & 2.27 (0.04) & 2.31 (0.02) & 2.34 (0.01) \\
AIPWE & 2.13 (0.14) & 2.23 (0.09) & 2.27 (0.05) & 2.31 (0.03) & 2.32 (0.03) \\
CNN &  2.12 (0.10) & 2.21 (0.07) & 2.28 (0.04) & 2.31 (0.03) & 2.34 (0.02) \\
ACNN &  \textbf{2.18 (0.12)} & \textbf{2.27 (0.08)} & \textbf{2.32 (0.04)} & \textbf{2.36 (0.03)} & \textbf{2.38 (0.02)}  \\
\bottomrule
\end{tabular}
\begin{tablenotes}
\item $\ell_1$-PLS, $\ell_1$ penalized least squares; Q-RF, Q-learning using random forests; AIPWE, augmented inverse probability weighted estimation; CNN, causal $k$-nearest neighbor; ACNN, adaptive causal $k$-nearest neighbor.
\end{tablenotes}
\end{threeparttable}
\end{table}

\begin{table}
\centering
\caption{Mean (std) of empirical value functions evaluated on on the test set for Scenarios 1-3 when the dimension is moderate ($p=25$) and covariates 3-25 are correlated. The best value function for each scenario and sample size combination is in bold.}
\label{tab:corrvalueknnhigh}
\begin{threeparttable}
\begin{tabular}{lccccc}
\addlinespace
\toprule
& $n=50$ & $n=100$ & $n=200$ & $n=400$ & $n=800$ \\
\midrule
\multicolumn{6}{c}{Scenario 1 (Optimal value $2.09$)} \\
\midrule
$\ell_1$-PLS & \textbf{1.90 (0.13)} & \textbf{1.99 (0.08)} & \textbf{2.03 (0.04)} & \textbf{2.05 (0.02)} & \textbf{2.06 (0.01)} \\
Q-RF & 1.89 (0.11) & 1.95 (0.06) & 2.00 (0.04) & 2.03 (0.02) & 2.04 (0.01) \\
AIPWE & 1.83 (0.13) & 1.91 (0.10) & 1.97 (0.08) & 2.02 (0.05) & 2.03 (0.03) \\
CNN & 1.87 (0.13) & 1.92 (0.08) & 1.95 (0.06) & 1.98 (0.04) & 2.00 (0.03)  \\
ACNN & 1.84 (0.12) & 1.88 (0.11) & 1.95 (0.09) & 2.01 (0.06) & 2.05 (0.03) \\
\midrule
\multicolumn{6}{c}{Scenario 2 (Optimal value $1.95$)} \\
\midrule
$\ell_1$-PLS & 1.50 (0.09) & 1.54 (0.09) & 1.58 (0.07) & 1.62 (0.05) & 1.64 (0.03) \\
Q-RF & 1.53 (0.11) & 1.61 (0.09) & 1.72 (0.07) & 1.81 (0.05) & 1.87 (0.03)  \\
AIPWE &   1.54 (0.12) & 1.62 (0.13) & 1.74 (0.10) & 1.83 (0.06) & 1.88 (0.03) \\
CNN & 1.57 (0.11) & 1.62 (0.08) & 1.67 (0.05) & 1.71 (0.04) & 1.75 (0.03   \\
ACNN & \textbf{1.58 (0.12)} & \textbf{1.68 (0.13)} & \textbf{1.78 (0.11)} & \textbf{1.86 (0.06)} & \textbf{1.90 (0.03)} \\
\midrule
\multicolumn{6}{c}{Scenario 3 (Optimal value $2.41$)} \\
\midrule
$\ell_1$-PLS & 1.88 (0.02) & 1.88 (0.02) & 1.88 (0.02) & 1.88 (0.02) & 1.88 (0.02)  \\
Q-RF & 1.95 (0.06) & 2.02 (0.07) & 2.11 (0.06) & 2.19 (0.05) & 2.27 (0.03)  \\
AIPWE &  1.97 (0.10) & 2.09 (0.12) & 2.22 (0.09) & 2.29 (0.04) & 2.31 (0.03)  \\
CNN &  2.01 (0.07) & 2.07 (0.06) & 2.12 (0.05) & 2.16 (0.03) & 2.19 (0.02)  \\
ACNN & \textbf{2.08 (0.13)} & \textbf{2.20 (0.12)} & \textbf{2.30 (0.06)} & \textbf{2.35 (0.04)} & \textbf{2.38 (0.02)}  \\
\bottomrule
\end{tabular}
\begin{tablenotes}
\item $\ell_1$-PLS, $\ell_1$ penalized least squares; Q-RF, Q-learning using random forests; AIPWE, augmented inverse probability weighted estimation; CNN, causal $k$-nearest neighbor; ACNN, adaptive causal $k$-nearest neighbor.
\end{tablenotes}
\end{threeparttable}
\end{table}

\bibliographystyle{jasa}
\bibliography{myreference}

\end{document}